\documentclass{article}

\usepackage{txie}
\usepackage{soul}

\usepackage[final]{neurips_2022}

\usepackage[utf8]{inputenc} %
\usepackage[T1]{fontenc}    %
\usepackage{hyperref}       %
\usepackage{url}            %
\usepackage{booktabs}       %
\usepackage{amsfonts}       %
\usepackage{nicefrac}       %
\usepackage{microtype}      %
\usepackage{xcolor}         %

\usepackage{longtable}

\usepackage[scaled=.91]{helvet}

\title{Interaction-Grounded Learning\\ with Action-Inclusive Feedback}

\usepackage[scaled=.99]{inconsolata}
\author{%
  Tengyang Xie\thanks{equal contribution}\\
  UIUC\\
  \href{mailto:tx10@illinois.edu}{\texttt{tx10@illinois.edu}}\\
  \And
  Akanksha Saran\footnotemark[1]\\
  Microsoft Research, NYC\\
  \href{mailto:akanksha.saran@microsoft.com}{\texttt{akanksha.saran@microsoft.com}}\\
  \And
  Dylan J. Foster\\
  Microsoft Research, New England\\
  \href{mailto:dylanfoster@microsoft.com}{\texttt{dylanfoster@microsoft.com}}\\
  \And
  Lekan Molu\\
  Microsoft Research, NYC\\
  \href{mailto:lekanmolu@microsoft.com}{\texttt{lekanmolu@microsoft.com}}\\
  \And
  Ida Momennejad \\
  Microsoft Research, NYC\\
  \href{mailto:idamo@microsoft.com}{\texttt{idamo@microsoft.com}}\\
  \And
  Nan Jiang\\
  UIUC\\
  \href{mailto:nanjiang@illinois.edu}{\texttt{nanjiang@illinois.edu}}\\
  \And
  Paul Mineiro\\
  Microsoft Research, NYC\\
  \href{mailto:pmineiro@microsoft.com}{\texttt{pmineiro@microsoft.com}}\\
  \And
  John Langford\\
  Microsoft Research, NYC\\
  \href{mailto:jcl@microsoft.com}{\texttt{jcl@microsoft.com}}\\
}

\newcommand{\oldIGL}{IGL (full CI)\xspace}
\newcommand{\newIGL}{AI-IGL\xspace}

\begin{document}

\maketitle

\begin{abstract}
Consider the problem setting of Interaction-Grounded Learning (IGL), in which a learner's goal is to optimally interact with the environment with no explicit reward to ground its policies. The agent observes a context vector, takes an action, and receives a feedback vector---using this information to effectively optimize a policy with respect to a latent reward function.
Prior analyzed approaches fail when the feedback vector contains the action, which significantly limits IGL's success in many potential scenarios such as Brain-computer interface (BCI) or Human-computer interface (HCI) applications.
We address this by creating an algorithm and analysis which allows IGL to work even when the feedback vector contains the action, encoded in any fashion. We provide theoretical guarantees and large-scale experiments based on supervised datasets to demonstrate the effectiveness of the new approach.
\end{abstract}

\section{Introduction}

Most real-world learning problems, such as BCI and HCI problems, are not tagged with rewards.  Consequently, (biological and artificial) learners must infer rewards based on interactions with the environment, which reacts to the learner's actions by generating feedback, but does not provide any explicit reward signal.
This paradigm has been previously studied by researchers~\citep[e.g.,][]{grizou2014calibration,nguyen2021interactive}, including a recent formalization~\citep{xie2021interaction} that proposed the term Interaction-Grounded Learning (IGL).

In IGL, the learning algorithm discovers a grounding for the feedback which implicitly discovers a reward function. An information-theoretic impossibility argument indicates additional assumptions are necessary to succeed. \citet{xie2021interaction} proceed by assuming the action is conditionally independent of the feedback given the reward. However, this is unnatural in many settings such as neurofeedback in BCI~\citep{katyal2014collaborative,mishra2015closed,debettencourt2015closed,munoz2020delineating,akinola2020accelerated,poole2021towards,xu2021accelerating} and multimodal interactive feedback in HCI~\citep{pantic2003toward,vitense2003multimodal,freeman2017multimodal,mott2017improving,duchowski2018gaze,saran2018human,saran2020understanding,zhang2020human,cui2021empathic,gao2021x2t}
where the action proceeds and thus influences the feedback. 
For example, if the feedback is an fMRI recording from a human brain implant in the parietal cortex, and the action of the algorithm is to display a number being imagined by the human (Fig.~\ref{fig:bci}), the parietal cortex feedback will contain information about both perceiving the number being displayed (action) and human's reaction to it for match or mismatch (reward). Another example is determining if a patient in a coma is happy or not, and using their brain signals (feedback) to dissociate the fact that they were asked about being happy (action) and the actual thought that they are happy (reward). An HCI example is a self-calibrating eye tracker used for typing on a screen by an ALS patient~\citep{mott2017improving,gao2021x2t}, where the automated agent calibrates the eye gaze to a certain point on a screen (action) and the human's gaze patterns (feedback) convey a response to the tracker's offset (action) as well as their satisfaction for being able to type a certain key (reward).
If you apply %
the prior IGL approach to such settings, 
it will fail catastrophically because the requirement of conditional independence is essential to its function. This motivates the question:
\begin{quote}\em\centering
    Is it possible to do interaction-ground learning when the feedback has the full information of the action embedded in it?
\end{quote}

We propose a new approach to solve IGL, which we call action-inclusive IGL (AI-IGL), that allows the action to be incorporated into the feedback in arbitrary ways. %
We consider latent reward as playing the role of latent states, which can be further separated via a contrastive learning method (Section~\ref{sec:latent_stat_dis}). Different from the typical latent state discovery in rich-observation reinforcement learning~\citep[e.g.,][]{dann2018oracle,du2019provably,misra2020kinematic}, the IGL setting also requires identifying the semantic meaning of the latent reward states, which is addressed by a symmetry breaking procedure (Section~\ref{sec:sym_brk}). We analyze the theoretical properties of the proposed approach, and we prove that it is guaranteed to learn a near-optimal policy as long as the feedback  satisfies a weaker context conditional independence assumption.  We also empirically evaluate the proposed AI-IGL approach using large-scale experiments on Open-ML's supervised classification datasets~\citep{bischl2021openml}, 
as well as a simulated BCI experiment using real human fMRI data~\citep{ellis2020facilitating}, demonstrating the effectiveness of the proposed approach (Section~\ref{sec:experiments}).
Thus, our findings broaden the scope of applicability for IGL. %
The paper proceeds as follows. In Section~\ref{sec:background}, we present the mathematical formulation for IGL. In Section~\ref{sec:solution_cpt}, we present a contrastive learning perspective for grounding latent reward which helps to expand the applicability of IGL. In Section~\ref{sec:algo_theory}, we state the resulting algorithm AI-IGL. We provide experimental support for the technique in Section~\ref{sec:experiments} using a diverse set of supervised learning datasets %
and a simulated BCI dataset.
We conclude with discussion in Section~\ref{sec:discussion}.

\begin{figure}
    \centering
    \includegraphics[width=0.8\textwidth]{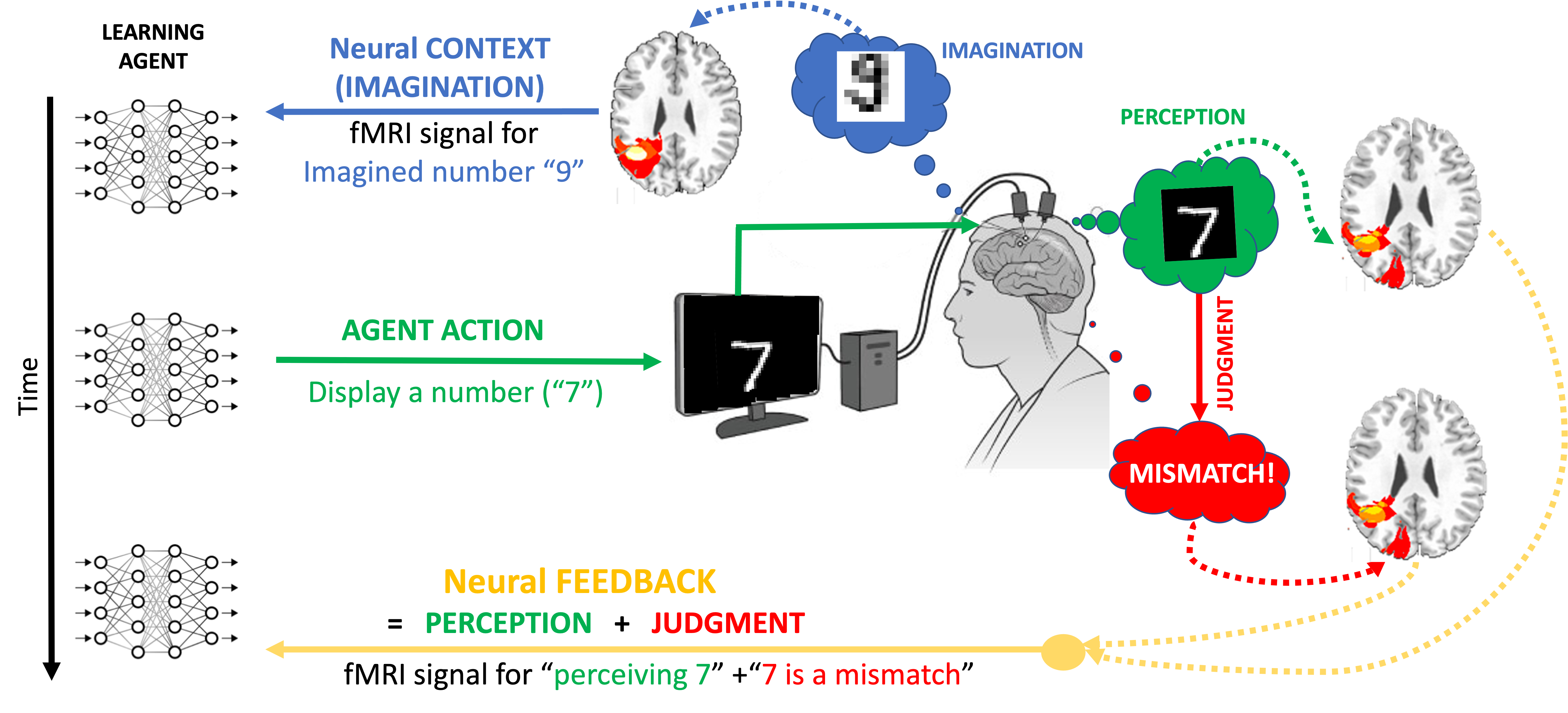}
    \vspace{-5pt}
    \caption{A simulated BCI interaction experiment for a number-guessing task. The human imagines a number, which the agent has to identify through their fMRI signals. The agent displays the number it identifies. The human perceives this number (action) and judges whether it is a match or not (reward)--the mixing of these two activities creates a complex feedback signal, which the learning agent has to decode through the interaction.}
    \label{fig:bci}
    \vspace{-10pt}
\end{figure}

\section{Background}
\label{sec:background}
\vspace{-2mm}
\paragraph{Interaction-Grounded Learning}
This paper studies the {\em Interaction-Grounded Learning} (IGL) setting~\citep{xie2021interaction}, where the learner optimizes for a latent reward by interacting with the environment and associating (``grounding'') observed feedback with the latent reward.  At each time step, the learner receives an i.i.d.~ context $x$ from context set $\Xcal$ and distribution $d_0$. The learner then selects an action $a \in \Acal$ from a finite action set $|\Acal| = K$. The environment generates a latent binary reward $r \in \{0,1\}$ (can be either deterministic or stochastic) and a feedback vector $y \in \Ycal$ conditional on $(x,a)$, but only the feedback vector $y$ is revealed to the learner. In this paper, we use $R(x,a) \coloneqq \E_{x,a}[r]$ to denote the expected (latent) reward after executing action $a$ on context $x$. The space of context $\Xcal$ and feedback vector $\Ycal$ can be arbitrarily large.

Throughout this paper, we use $\pi \in \Pi: \Xcal \to \Delta(\Acal)$ to denote a (stochastic) policy. The expected return of policy $\pi$ is defined by $V(\pi) = \E_{(x,a) \sim d_0 \times \pi} [r]$. The learning goal of IGL is to find the optimal policy in the policy class,
$
\pi^\star = \argmax_{\pi \in \Pi} V(\pi),
$
only from the observations of context-action-feedback tuples, $(x,a,y)$.
This paper mainly considers the batch setting, and we use $\mu$ to denote the behavior policy.
In this paper, we also introduce value function classes and decoder classes. We assume the learner has access to a value function class $\Fcal$ where $f \in \Fcal: \Xcal \times \Acal \to [0,1]$ and reward decoder class $\psi \in \Psi: \Ycal \times \Acal \to [0,1]$. We defer the assumptions we made on these classes to Section~\ref{sec:solution_cpt} for clarity.

We may hope to solve IGL without any additional assumptions.
However, it is information-theoretically impossible without additional assumptions, even if the latent reward is decodable from $(x,a,y)$, as demonstrated by the following example.
\begin{example}[Hardness of assumption-free IGL]
Suppose $y = (x,a)$ and suppose the reward is deterministic in $(x, a)$. In this case, the latent reward $r$ can be perfectly decoded from $y$. However, the learner receives no more information than $(x,a)$ from the $(x,a,y)$ tuple.  Thus if $\Pi$ contains at least 2 policies, for any environment where any IGL algorithm succeeds, we can construct another environment with the same observable statistics where that algorithm must fail.
\label{example:hard}
\end{example}

\paragraph{IGL with full conditional independence}
Example~\ref{example:hard} demonstrates the need for further assumptions to succeed at IGL. \citet{xie2021interaction} proposed an algorithm that leverages the following conditional independence assumption to facilitate grounding the feedback in the latent reward. \begin{assumption}[Full conditional independence]
\label{asmp:full_CI}
For arbitrary $(x,a,r,y)$ tuples where $r$ and $y$ are generated conditional on the context $x$ and action $a$, we assume the feedback vector $y$ is conditionally independent of context $x$ and action $a$ given the latent reward $r$, i.e. $x,a \indep y | r$.
\end{assumption}
\citet{xie2021interaction} introduce a reward decoder class $\psi \in \Psi: \Ycal \to [0,1]$ for estimating $\E[r|y]$, which leads to the decoded return $V(\pi,\psi) \coloneqq \E_{(x,a) \sim d_0 \times \pi} [\psi(y)]$. They proved that it is possible to learn the best $\pi$ and $\psi$ jointly by optimizing the following proxy learning objective:
\begin{align}
\label{eq:old_igl_objective}
\argmax_{(\pi,\psi) \in \Pi \times \Psi} \Jcal(\pi,\psi) \coloneqq V(\pi,\psi) - V(\pi_\bad, \psi),
\end{align}
where $\pi_\bad$ is a policy known to have low expected return. Over this paper, we use \oldIGL to denote the proposed algorithm of~\citet{xie2021interaction}.

\section{A Contrastive-learning Perspective for Grounding Latent Reward}
\label{sec:solution_cpt}

The existing work of interaction-grounded learning leverages the assumption of full conditional independence (Assumption~\ref{asmp:full_CI}), where the feedback vector only contains information from the latent reward.
The goal of this paper is to relax these constraining assumptions and broaden the scope of applicability for IGL.
This paper focuses on the scenario where the action information is possibly embedded in the feedback vector, which is formalized by the following assumption.
\begin{assumption}[Context Conditional Independence]
\label{asmp:x_CI}
For arbitrary $(x,a,r,y)$ tuple where $r$ and $y$ are generated conditional on the context $x$ and action $a$, we assume the feedback vector $y$ is conditionally independent of context $x$ given the latent reward $r$ and action $a$. That is,  $x \indep y | a, r$.
\end{assumption}

Assumption~\ref{asmp:x_CI} allows the feedback vector $y$ to be generated from both latent reward $r$ and action $a$, differing from Assumption~\ref{asmp:full_CI} which constrains the feedback vector $y$ to be generated based only on the latent reward $r$.  We discuss the implications at the end of this section.

\subsection{Grounding Latent Reward via Contrastive Learning}
\label{sec:latent_stat_dis}

In this section, we propose a contrastive learning objective for interaction-grounded learning, which further guides the design of our algorithm. We perform derivations with exact expectations for clarity and intuitions and provide finite-sample guarantees in Section~\ref{sec:algo_theory}.

\begin{assumption}[Separability]
\label{asmp:realizability}
For each $\abar \in \Acal$, there exists an $(f_\abar^\star,\psi_\abar^\star) \in \Fcal \times \Psi$, such that: 1)
$\E\left[ \psi_\abar^\star(y,\abar) | \abar, r = 1 \right] - \E\left[ \psi_\abar^\star(y,\abar) | \abar, r = 0 \right] = 1$; 2) $|\E_\mu\left[ f_\abar^\star(x,\abar) | \abar, r = 1 \right] - \E_\mu\left[ f_\abar^\star(x,\abar) | \abar, r = 0 \right]| \geq \Delta_\Fcal$.
We also assume $1 - \Psi \subseteq \Psi$, where $1 - \Psi \coloneqq \{1 - \psi(\cdot,\cdot):\psi \in \Psi\}$.
\end{assumption}

Assumption~\ref{asmp:realizability} consists of two components. For
$\Psi$, it is a realizability assumption that ensures that a perfect reward decoder is included in the function classes.   Although this
superficially appears unreasonably strong, note $y$ is generated based
upon $a$ and the \emph{realization} of $r$; therefore this is compatible
with stochastic rewards.
For $\Fcal$, it ensures the expected predicted reward conditioned on the latent reward, having value $r \in \{0,1\}$ and the action being $a$, is separable.
When $\mu = \pi_\unif$, $\Delta_\Fcal$ can be lower bounded by $\max_{f \in \Fcal} 4 |\Cov_{\pi_\abar}(f,R)|$ ($\pi_\abar$ denotes the constant policy with action $\abar$). Detailed proof of this argument can be found in Appendix~\ref{sec:proofs}. One sufficient condition is that the expected latent reward $R(\cdot,\abar)$ has enough variance and $R \in \Fcal$.
The condition of $1 - \Psi \subseteq \Psi$ can be constructed easily via standard realizable classes. That is, if $\psi^\star_\abar \in \Psi'$ for some classes $\Psi'$, simply setting $\Psi \leftarrow \Psi' \bigcup (1 - \Psi')$ satisfies Assumption~\ref{asmp:realizability}. Note that, this construction of $\Psi$ only amplifies the size of $\Psi'$ by a factor of 2.

\paragraph{Reward Prediction via Contrastive Learning}
We now formulate the following contrastive-learning objective for solving IGL.
Suppose $\mu(a|x)$ is the behavior policy. We also abuse $\mu(x,a,y)$ to denote the data distribution, and $\mu_a(x,y)$, $\mu_a(x)$, $\mu_a(y)$ to denote the marginal distribution under action $a \in \Acal$.
We construct an augmented data distribution for each $a \in \Acal$: $\mutilde_a(x,y) \coloneqq \mu_a(x) \cdot \mu_a(y)$ (i.e., sampling $x$ and $y$ independently from $\mu_a$).

Conceptually, for each action $a \in \Acal$, we consider tuples $(x,y) \sim \mu_a(x,y)$ and $(\xtilde,\ytilde) \sim \mutilde_a(x,y)$. From Assumption~\ref{asmp:realizability}, conditioned on $(x, a)$ 
the optimal feedback decoder $\psi^\star(y, a)$ has mean equal to the optimal reward predictor $f^\star(x, a)$. Therefore we might seek an $(f, \psi)$ pair which minimizes any consistent loss function, e.g., squared loss. However this is trivially achievable, e.g., by having both always predict 0.  Therefore we formulate a contrastive-learning objective, where we maximize the loss between the predictor and decoder on the augmented data distribution.
For each $a \in \Acal$, we solve the following objective
\begin{align}
\label{eq:def_obj_a}
\argmin_{(f_a,\psi_a) \in \Fcal \times \Psi} \Lcal_a(f_a,\psi_a) \coloneqq \E_{\mu_a}\left[ \left(f_a(x,a) - \psi_a(y,a)\right)^2 \right] - \E_{\mutilde_a}\left[ \left(f_a(\xtilde,a) - \psi_a(\ytilde,a)\right)^2 \right].
\end{align}
In the notation of equation~\eqref{eq:def_obj_a}, the $a$ subscript indicates the $(f_a, \psi_a)$ pair are optimal for action $a$.  Note they are always evaluated at $a$, which we retain as an input for compatibility with the original function classes.

Note that $\Lcal$ is also similar to many popular contrastive losses~\citep[e.g.,][]{ wu2018unsupervised, chen2020simple} especially for the spectral contrastive loss~\citep{haochen2021provable}.
\begin{align}
\Lcal_a(f_a,\psi_a) = &~ \E_{\mu_a}\left[ f_a(x,a)^2 - 2 f_a(x,a) \psi_a(y,a) + \psi_a(y,a)^2 \right]
\\
&~ - \E_{\mutilde_a}\left[ f_a(x,a)^2 - 2 f_a(x,a) \psi_a(y,a) + \psi_a(y,a)^2 \right]
\\
= &~ - 2 \left( \E_{\mu_a}\left[f_a(x,a) \psi_a(y,a)\right] - \E_{\mutilde_a}\left[ f_a(x,a) \psi_a(y,a) \right] \right)
\tag{\bf spectral contrastive loss}
\\
= &~ - 2 \left( \E_{\mu_a}\left[f_a(x,a) \psi_a(y,a)\right] - \E_{\mu_a}\left[ f_a(x,a) \right] \E_{\mu_a}\left[ \psi_a(y,a) \right] \right).
\tag{$\mutilde_a(x,y) = \mu_a(x) \cdot \mu_a(y)$}
\end{align}
Below we show that minimizing $\Lcal_a(f_a,\psi_a)$  decodes the latent reward under Assumptions~\ref{asmp:x_CI} and~\ref{asmp:realizability} up to a sign ambiguity. For simplicity, we introduce the notation of $f_{a,r}$ and $\psi_{a,r}$ for any $(f, \psi) \in \Fcal \times \Psi$,
\begin{align}\textstyle
\label{eq:def_far_psiar}
    f_{a,r} \coloneqq \sum_{x} \Pr(x|a,r) f(x,a), \quad \psi_{a,r} \coloneqq \sum_{y} \Pr(y|a,r) \psi(y,a).
\end{align}
$f_{a,r}$ and $\psi_{a,r}$ are the expected predicted reward of $f(x,a)$ and decoded reward $\psi(y,a)$ (under the behavior policy $\mu$ for $f_{a,r}$) conditioned on the latent reward having value $r$ and the action being $a$.

\begin{proposition}
\label{prop:decoding}
For any action $a \in \Acal$, if $\mu(a|x) > 0$ for all $x \in \Xcal$ and Assumption~\ref{asmp:x_CI} and~\ref{asmp:realizability} hold, and let $(\fhat_a, \psihat_a)$ be the solution of~\Eqref{eq:def_obj_a}. Then, $|\fhat_{a,1} - \fhat_{a,0}| = \max_{f \in \Fcal} |f_{a,1} - f_{a,0}|$ and $|\psihat_{a,1} - \psihat_{a,0}| = \max_{\psi \in \Psi} |\psi_{a,1} - \psi_{a,0}|$.
\end{proposition}
\begin{proof}[Proof Sketch]
For any policy $\pi$, we use $d^\pi_a \coloneqq \sum_{x}d_0(x) \pi(a|x)$ to denote the visitation occupancy of action $a$ under policy $\pi$, and $\rho^\pi_a \coloneqq \nicefrac{1}{d^\pi_a}\sum_{x}d_0(x) \pi(a|x) R(x,a)$ to denote the average reward received under executing action $a$.
Then, by the context conditional independent assumption $x \indep y | r,a$ (Assumption~\ref{asmp:x_CI}), we know
\begin{align}\textstyle
\E_{\mu_a}\left[f(x) \psi(y)\right] = &~ \textstyle (1 - \rho^\mu_a ) f_{a,0} \psi_{a,0} + \rho^\mu_a f_{a,1} \psi_{a,1}
\\
\textstyle \E_{\mu_a}\left[ f(x) \right] \E_{\mu_a}\left[ \psi(y) \right] = &~ \textstyle (1 - \rho^\mu_a)^2 f_{a,0} \psi_{a,0} + \left( \rho^\mu_a \right)^2 f_{a,1} \psi_{a,1} + (1 - \rho^\mu_a) \rho^\mu_a (f_{a,0} \psi_{a,1} + f_{a,1} \psi_{a,0})
\\
\textstyle \Longrightarrow \Lcal_a(f_a,\psi_a) \propto &~ \textstyle - (f_{a,1} - f_{a,0}) (\psi_{a,1} - \psi_{a,0}).
\end{align}
Therefore, separately maximizing $|f_{a,1} - f_{a,0}|$ and $|\psi_{a,1} - \psi_{a,0}|$ maximizes $\Lcal_a(f_a,\psi_a)$.
\end{proof}

\subsection{Symmetry Breaking}
\label{sec:sym_brk}

In the last section, we demonstrated that the latent reward could be decoded in a contrastive-learning manner up to a sign ambiguity.  The following example demonstrates the ambiguity.

\begin{example}[Why do we need extra information to identify the latent reward?]
\label{exp:symmetry}
In the optimal solution of 
objective~\Eqref{eq:def_obj_a}, both $\psihat_a$ and $\psihat'_a \coloneqq 1 - \psihat_a$ yield the same value.
It is information-theoretically difficult to distinguish which one of them is the correct solution without extra information.
That is because, for any environment {\sf ENV1}, there always exists a ``symmetric'' environment {\sf ENV2}, where: 1) $R(x,a)$ of {\sf ENV1} is identical to $(1 - R(x,a))$ of {\sf ENV2} for all $(x,a) \in \Xcal \times \Acal$; 2) the the conditional distribution of $y|r,a$ in {\sf ENV1} is identical to the conditional distribution of $y|1-r,a$ in the {\sf ENV2} for all $a \in \Acal$.
In this example, {\sf ENV1} and {\sf ENV2} will always generate the identical distribution of feedback vector $y$ after any $(x,a) \in \Xcal \times \Acal$. However, {\sf ENV1} and {\sf ENV2} have the exactly opposite latent reward information.
\end{example}

As we demonstrate in Example~\ref{exp:symmetry}, the learner decoder from~\Eqref{eq:def_obj_a} could be corresponding to a symmetric pair of semantic meanings, and identifying them without extra information is information-theoretically impossible.
The {\em symmetry breaking} procedure is one of the key challenges of interaction-grounded learning. To achieve symmetry breaking, we make the following assumption to ensure the identifiability of the latent reward.
\begin{assumption}[Baseline Policies]
\label{asmp:pi_bad}
For each $a \in \Acal$, there exists a baseline policy $\pi_\bad^a$, such that,
\begin{enumerate}[(a)]
\item $\pi_\bad^a$ satisfies $\sum_{x}d_0(x) \pi(a|x) \geq c_m > 0$.
\item $|\frac{1}{2} - \rho^{\pi_\bad^a}_a| \geq \eta$, where $\rho^{\pi_\bad^a}_a = \frac{\sum_{x}d_0(x) \pi_\bad^a(a|x) R(x,a)}{\sum_{x}d_0(x) \pi_\bad^a(a|x)}$.
\end{enumerate}
\end{assumption}

To instantiate Assumption~\ref{asmp:pi_bad} in practice, we provide the following simple example of $\pi_\bad$ that satisfies Assumption~\ref{asmp:pi_bad}.
Suppose $\pi_\bad = \pi_\unif$ (uniformly random policy), and we have ``all constant policies are bad'', i.e., $V(\pi_\abar = \1(a = \abar)) < \nicefrac{1}{2} - \eta$ for all $\abar \in \Acal$. Then it is easy to verify that $c_m = \nicefrac{1}{K}$ and $\rho_\abar^{\pi_\bad} \leq \nicefrac{1}{2} - \eta$ for all $\abar \in \Acal$.

Note that $\pi_\bad^a$ can be different over actions. Intuitively, Assumption~\ref{asmp:pi_bad}(a) is saying that the total probability of $\pi_\bad^a$ selecting action $a$ (over all context $x \in \Xcal$) is at least $c_m$. This condition ensures that $\pi_\bad^a$ has enough visitation to action $a$ and makes symmetry breaking possible. Assumption~\ref{asmp:pi_bad}(b) states that if we only consider the reward obtained from taking action $a$, $\pi_\bad^a$ is known to be either ``sufficiently bad'' or ``sufficiently good''. Note the directionality of the extremeness of $\pi_\bad^a$ must be known, e.g., a policy which has a unknown reward of either 0 or 1 is not usable.  This condition follows a similar intuition as the identifiability assumption of~\citet[Assumption 2]{xie2021interaction} and breaks the symmetry. 
For example, consider the {\sf ENV1} and {\sf ENV2} introduced in Example~\ref{exp:symmetry}, $\rho^{\pi}_{a,{\sf ENV1}} = 1 - \rho^{\pi}_{a,{\sf ENV2}}$ for any policy $\pi$. To separating {\sf ENV1} and {\sf ENV2} using some policy $\pi$, $\rho^{\pi}_{a,{\sf ENV1}}$ and $\rho^{\pi}_{a,{\sf ENV2}}$ require to have a non-zero gap, which leads to Assumption~\ref{asmp:pi_bad}(b).

The effectiveness of symmetry breaking under Assumption~\ref{asmp:pi_bad} can be summarized as below: we conduct the following estimation of $\rho^{\pi_\bad}_a$, using the learned $\psihat_a$,
$
\rhohat^{\pi_\bad^a}_a = \frac{\sum_{x}d_0(x) \pi_\bad^a(a|x) \psihat_a(x,a)}{\sum_{x}d_0(x) \pi_\bad^a(a|x)}.
$
If $\psihat_a$ can efficiently decode the latent reward, then $\rhohat^{\pi_\bad^a}_a$  converges to either $\rho^{\pi_\bad}_a$ or $1 - \rho^{\pi_\bad}_a$. Therefore, applying Assumption~\ref{asmp:pi_bad}(b) breaks the symmetry.

\subsection{Comparison to Full CI}

When we have the context conditional independence, it is easy to verify the failure of optimizing the original IGL objective~\Eqref{eq:old_igl_objective} by the following example. 

\begin{example}[Failure of the original IGL objective under Assumption~\ref{asmp:x_CI}]
Let $\Xcal = \Acal = \{1,2,\dotsc,10\}$ and feedback vector is generated by $y = (a + R(x,a)) \mod 10$ (we use $\%$ to denote ${\rm mod}$ in the following part). We also assume $d_0(x) = \pi_\bad(a|x) = \nicefrac{1}{10}$ for any $(x,a) \in \Xcal \times \Acal$ and $R(x,a) = \pi^\star(a|x) \coloneqq \1(x = a)$. Then, we have, for any $\psi: \Ycal \to [0,1]$ (approach proposed by~\citet{xie2021interaction} assumes the reward decoder only takes feedback vector $y$ as the input),
\begin{align*}
\textstyle
\Lcal(\pi^\star,\psi) = &~ \textstyle \frac{1}{10} \sum_{x=1}^{10} \sum_{a = 1}^{10} \1(x=a) \psi((a + 1) \% 10) - \frac{1}{100} \sum_{x=1}^{10} \sum_{a = 1}^{10} \psi((a + \1(x=a)) \% 10)
\\
= &~ \textstyle \frac{1}{10} \sum_{a = 1}^{10} \psi(a) - \frac{1}{10} \sum_{a = 1}^{10} \psi(a) = 0.
\end{align*}
On the other hand, consider the constant policy $\pi_{1}(a|x) \coloneqq \1(a = 1)$ for all $x \in \Xcal$ and decoder $\psi_2(y) \coloneqq \1(y = 2)$ for all $y \in \Ycal$, then,
\begin{align*}
\textstyle
\Lcal(\pi_{1},\psi_2) = &~ \textstyle \frac{1}{10} \sum_{x=1}^{10} \sum_{a = 1}^{10} \1(a=1) \psi_2((a + 1) \% 10) - \frac{1}{100} \sum_{x=1}^{10} \sum_{a = 1}^{10} \psi_2((a + \1(x=a)) \% 10)
\\
= &~ \textstyle \psi_2(2) - \frac{1}{10} \sum_{a = 1}^{10} \psi_2(a) = 0.9 > \Lcal(\pi^\star,\psi), ~~ \forall \psi \in \Psi.
\end{align*}
This implies that maximizing the original IGL objective~\Eqref{eq:old_igl_objective} could not always converge to $\pi^\star$ when we only have the context conditional independence.
\end{example}

This example indicates optimizing a combined contrastive objective with a single symmetry-breaking policy is insufficient to succeed in our $x \indep y | r,a$ case.  Our current approach corresponds to optimizing a contrastive objective and breaking symmetry for each action separately rather than simultaneously.

\subsection{Viewing Latent Reward as a Latent State}
\begin{figure*}[t!]
\centering
\begin{subfigure}{0.35\textwidth}
\centering
\includegraphics[width=0.57142857142\textwidth]{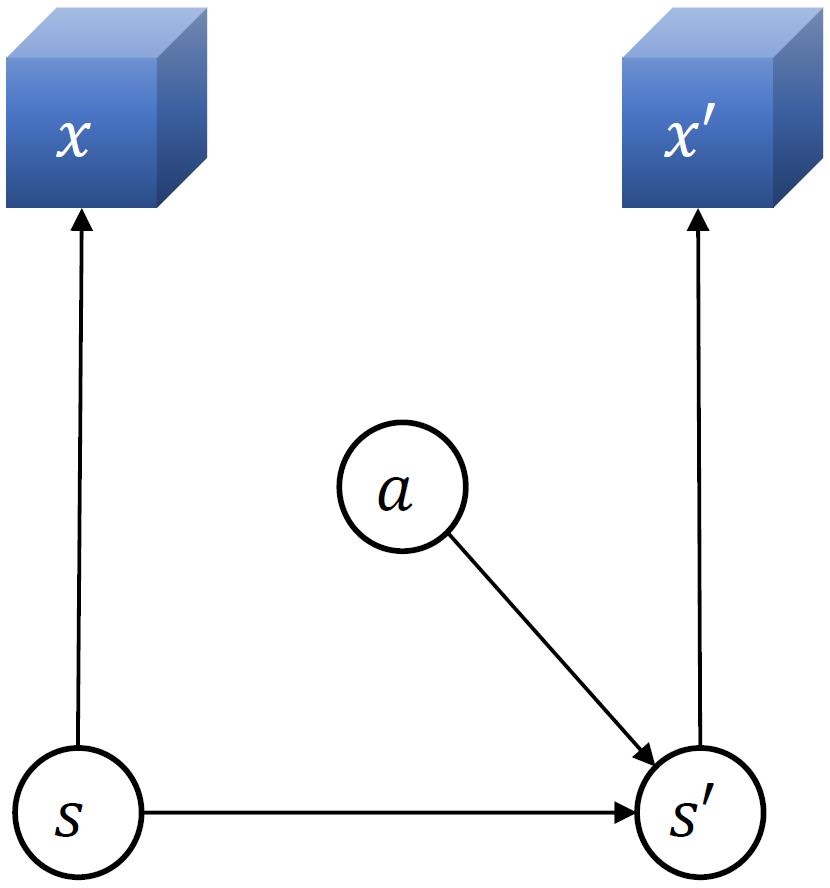}
\caption{Latent State in Rich-Observation RL~\citep{misra2020kinematic}}
\label{fig:homer_latent_states}
\end{subfigure}
~
\begin{subfigure}{0.3\textwidth}
\centering
\includegraphics[width=.95\textwidth]{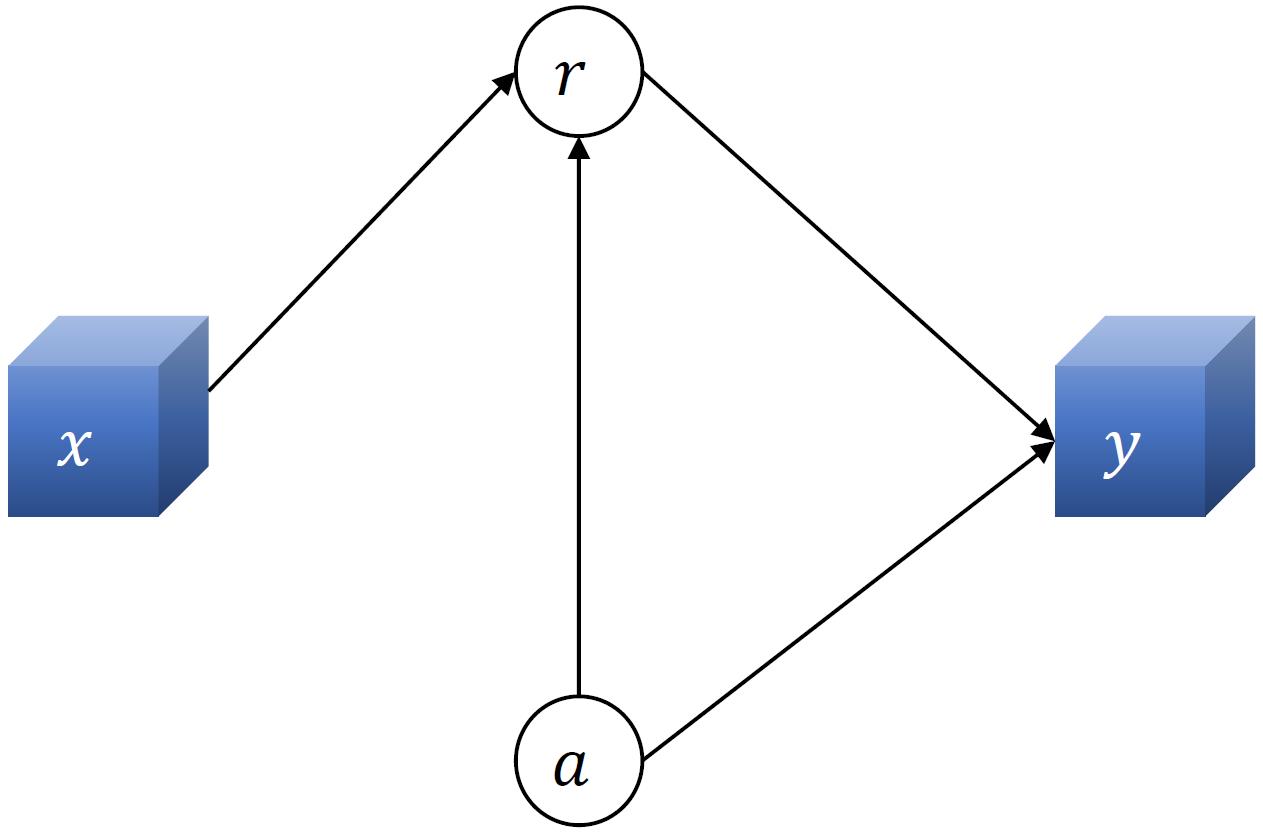}
\caption{Latent Reward in IGL under Assumption~\ref{asmp:x_CI} (this paper)}
\label{fig:igl_latent_states}
\end{subfigure}
~
\begin{subfigure}{0.3\textwidth}
\centering
\includegraphics[width=.95\textwidth]{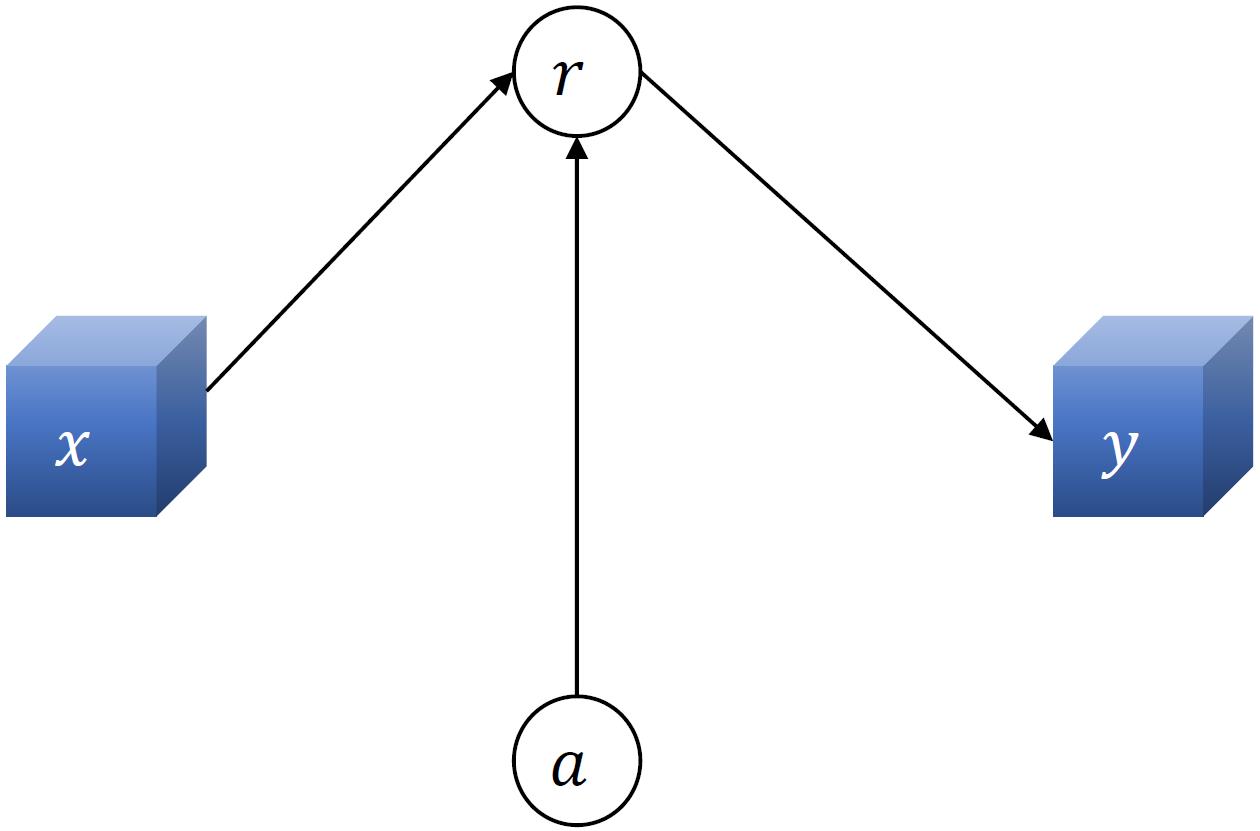}
\caption{Latent Reward in IGL under Assumption~\ref{asmp:full_CI}~\citep{xie2021interaction}}
\label{fig:igl_latent_states_old}
\end{subfigure}
\caption{Causal graphs of interaction-grounded learning under different assumptions as  well as rich-observation reinforcement Learning.}
\label{fig:Causal_Graphs}
\vspace{-3mm}
\end{figure*}
Our approach is motivated by 
latent state discovery in Rich-Observation RL~\citep{misra2020kinematic}.  Figure~\ref{fig:Causal_Graphs} compares the causal graphs of Rich-Observation RL, IGL with context conditional independence, and IGL with full conditional independence.  In Rich-Observation RL, a contrastive learning objective is used to discover latent states; whereas in IGL a contrastive learning objective is used to discover latent rewards.  In this manner we view latent rewards analogously to latent states.

Identifying latent states up to a permutation is completely acceptable in Rich-Observations RL, as the resulting imputed MDP orders policies identically to the true MDP.  However in IGL the latent states have scalar values associated with them (rewards), and identifying up to a permutation is not sufficient to order policies correctly.
Thus, we require the additional step of symmetry breaking.

\section{Main Algorithm}
\label{sec:algo_theory}
This section instantiates the algorithm for IGL with action-inclusive feedback, following the concept we introduced in Section~\ref{sec:solution_cpt}. For simplicity, we select the uniformly random policy $\pi_\unif$ as behavior policy $\mu$ in this section. That choice of $\mu$ can be further relaxed %
using an appropriate importance weight. We first introduce the empirical estimation of $\Lcal_a$ as follows ($\Lcal_a$ is defined in~\Eqref{eq:def_obj_a}. $\Jcal_{\abar,\Dcal}$ estimates the spectral contrastive loss, which corresponds to $-\Lcal_a$). For any $\abar \in \Acal$, we define the following empirical estimation of the spectral contrastive loss:
\begin{align}
\label{eq:emp_obj}
\Jcal_{\abar,\Dcal}(f,\psi) = \E_{\Dcal}\left[f(x,a) \psi(y,a) \1(a = \abar)\right] - \E_{\Dcal}\left[ f(x,a) \1(a = \abar) \right] \E_{\Dcal}\left[ \psi(y,a) \1(a = \abar) \right].
\end{align}

Using this definition, Algorithm~\ref{alg:IGL_x_CI} instantiates a version of the IGL algorithm with action-inclusive feedback. Without loss of generality, we also assume $\rho^{\pi_\bad^a}_a \leq \nicefrac{1}{2}$ for Assumption~\ref{asmp:pi_bad}(b) for all $a \in \Acal$ in this section. The case of $\rho^{\pi_\bad^a}_a > \nicefrac{1}{2}$ for some action $a$ can be addressed by modifying the symmetry-breaking step properly in Algorithm~\ref{alg:IGL_x_CI}.

\begin{algorithm}[th]
\caption{Action-inclusive IGL (AI-IGL)}
\label{alg:IGL_x_CI}
{\bfseries Input:} Batch data $\Dcal$ generated by $\mu = \pi_\unif$. baseline policy $\pi_\bad^{a \in \Acal}$.
\begin{algorithmic}[1]
\State Initialize policy $\pi_1$ as the uniform policy.
\For{$\abar \in \Acal$}
\State \label{line:lat_rwd_dis} Obtain $(f_{\abar},\psi_{\abar})$ by \algocmt{Latent State (Reward) Discovery}
\begin{align}
\label{eq:decode_f_psi}
(f_{\abar},\psi_{\abar}) \leftarrow \argmax_{(f,\psi) \in \Fcal \times \Psi} \Jcal_{\abar,\Dcal}(f,\psi),
\end{align}
where $\Jcal_{\abar,\Dcal}(f,\psi)$ is defined in \Eqref{eq:emp_obj}
\State \label{line:sym_brk1} Compute $\varrhohat^{\pi_\bad^{\abar}}_{\abar}$ by
$
\varrhohat^{\pi_\bad^{\abar}}_{\abar} = \frac{\sum_{(x,a,y) \in \Dcal}\pi_\bad^{\abar}(a|x) \psi_\abar(x,\abar) \1(a = \abar)}{\sum_{(x,a,y) \in \Dcal} \pi_\bad^{\abar}(a|x) \1(a = \abar)}.
$\algocmt{Symmetry Breaking}
\If{$\varrhohat^{\pi_\bad}_{\abar} > \frac{1}{2}$}  \label{step:sym_brk_1}
\quad $\psi'_{\abar} \leftarrow (1 - \psi_{\abar})$.
\Else \label{step:sym_brk_2}
\quad $\psi'_{\abar} \leftarrow \psi_{\abar}$.
\EndIf \label{line:sym_brk2}
\EndFor
\State Generate decoded contextual bandits dataset $\Dcal_\CB \leftarrow \{(x,a,\psi'_a(y,a),\mu(a|x)): (x,a,y) \in \Dcal \}$.
\State Output policy $\pihat(x) \leftarrow \CB(\Dcal_\CB)$, where $\CB$ denotes an offline contextual bandit oracle.
\end{algorithmic}
\end{algorithm}

At a high level, Algorithm~\ref{alg:IGL_x_CI} has two separate components, latent state (reward) discovery (line~\ref{line:lat_rwd_dis}) and symmetry breaking (line~\ref{line:sym_brk1}-\ref{line:sym_brk2}), for each action in $\Acal$.

\paragraph{Theoretical guarantees}

In Algorithm~\ref{alg:IGL_x_CI}, the output policy $\pihat$ is obtained by calling an offline contextual bandits oracle ($\CB$). We now formally define this oracle and its expected property.
\begin{definition}[Offline contextual bandits oracle]
\label{def:CB_oracle}
An algorithm $\CB$ is called an {\em offline contextual bandit oracle} if for any dataset $\Dcal= \{(x_i,a_i,r_i,\mu(a_i|x_i))\}_{i = 1}^{|\Dcal|}$ ($x_i \sim d_0, a_i \sim \mu$, and $r_i$ is the reward determined by $(x_i,a_i)$) and any policy class $\Pi$, the policy $\pihat$ produced by $\CB(\Dcal)$ satisfies
$
\varepsilon_{\CB} \coloneqq \max_{\pi \in \Pi} \E_{d_0 \times \pi}[r] - \E_{d_0 \times \pihat}[r]  \leq o(1).
$
\end{definition}

The notion in Definition~\ref{def:CB_oracle} corresponds to the standard policy learning approaches in the contextual bandits literature~\citep[e.g.,][]{langford2007epoch,dudik2011efficient,agarwal2014taming}, and typically leads to $\varepsilon_{\CB} = \sqrt{K \nicefrac{\log{\nicefrac{|\Pi|}{\delta}}}{|\Dcal|}}$.
We now provide the theoretical analysis of Algorithm~\ref{alg:IGL_x_CI}. In this paper, we use $d_{\Fcal, \Psi}$ to denote the joint statistical complexity of the class of $\Fcal$ and $\Psi$. For example, if the function classes are finite, we have $d_{\Fcal, \Psi} = \Ocal (\log \nicefrac{|\Fcal| |\Psi|}{\delta})$, and $\delta$ is the failure probability. The infinite function classes can be addressed by some advanced methods such as covering number or Rademacher complexity~\citep[see, e.g., ][]{mohri2018foundations}.
The following theorem provides the performance guarantee of the output policy of Algorithm~\ref{alg:IGL_x_CI}.
\begin{theorem}
\label{thm:main_thm}
Suppose Assumptions~\ref{asmp:x_CI},~\ref{asmp:realizability} and ~\ref{asmp:pi_bad} hold, $\Delta_\Fcal$ is defined in Assumption~\ref{asmp:realizability}, and $\pihat$ be the output policy of Algorithm~\ref{alg:IGL_x_CI}. If we have
$
|\Dcal| \geq \Ocal \left( \frac{K^3 d_{\Fcal, \Psi}}{\left(\min\{\eta \Delta_\Fcal, K c_m \}\right)^2} \right),
$
then, with high probability,\\
$
V(\pi^\star) - V(\pi)  \leq \Ocal \left( \frac{1}{\Delta_\Fcal} \sqrt{\frac{K^3 d_{\Fcal, \Psi}}{|\Dcal|}} \right) + \varepsilon_{\CB}.
$
\end{theorem}
Similar to the performance of~\citet{xie2021interaction}, the learned is guaranteed to converge in the right direction only after we have sufficient data for the symmetric breaking. The dependence on $K$ in Theorem~\ref{thm:main_thm} can be improved as different action has a separate learning procedure. For example, if we consider $\Fcal = \Fcal_1 \times \Fcal_2 \times \cdots \times \Fcal_K$ and $\Psi = \Psi_1 \times \Psi_2 \times \cdots \times \Psi_K$, where $\Fcal_a$ and $\Psi_a$ are independent components that is only corresponding to action $a$ (this is a common setup for linear approximated reinforcement learning approaches with discrete action space). If $\Fcal_a$ and $\Psi_a$ are identical copies of $K$ separate classes, we know $\log|\Fcal| = K \log|\Fcal_a|$ and $\log|\Psi| = K \log|\Psi_a|$, which leads a $\sqrt{K}$ improvement.

We now provide the proof sketch of Theorem~\ref{thm:main_thm} and we defer the detailed proof to Appendix~\ref{sec:proofs}.

\begin{proof}[Proof Sketch]
The proof of Theorem~\ref{thm:main_thm} consists of two different components---discovering latent reward and breaking the symmetry, which are formalized by the following lemma.
\begin{lemma}[Discovering latent reward]
\label{lem:ground_reward}
Suppose Assumptions~\ref{asmp:x_CI} and~\ref{asmp:realizability} hold, and let $(f_\abar, \psi_\abar)$ be obtained by \Eqref{eq:decode_f_psi}. Then, with high probability, we have
$|\psi_{\abar,1} - \psi_{\abar,0}| \geq \left( 1 - \Ocal \left( \frac{1}{\Delta_\Fcal} \sqrt{\frac{K^3 d_{\Fcal, \Psi}}{|\Dcal|}} \right) \right)$.
\end{lemma}
Lemma~\ref{lem:ground_reward} ensures that the learned decoder on~\Eqref{eq:decode_f_psi} correctly separates the latent reward. In particular since $\psi$ ranges over $[0, 1]$, Lemma~\ref{lem:ground_reward} ensures $\max\left(\mathrm{Pr}(\psi(y, a) = r), \mathrm{Pr}(\psi(y, a) = 1 - r)\right) > 1 - o(1)$ under the behaviour policy.  Thus, if we can break symmetry, we can use $\psi$ to generate a reward signal and reduce to ordinary contextual bandit learning. The following lemma guarantees the correctness of the symmetry-breaking step.
\begin{lemma}[Breaking symmetry]
\label{lem:estimation_per_action}
Suppose Assumption~\ref{asmp:pi_bad} holds. For any $\abar \in \Acal$, if we have
$
|\Dcal| \geq \Ocal \left( \frac{K^3 d_{\Fcal, \Psi}}{\left(\min\{\eta \Delta_\Fcal, K c_m \}\right)^2} \right),
$
then, with high probability, $\psi'_{\abar,1} \geq \psi'_{\abar,0}$.
\end{lemma}
Combining these two lemmas above as well as the $\CB$ oracle establishes the proof of Theorem~\ref{thm:main_thm}, and the detailed proof can be found in Appendix~\ref{sec:proofs}.
\end{proof}
\vspace{-3mm}

\section{Empirical Evaluations}
\label{sec:experiments}
In this section, we provide empirical evaluations in simulated environments created using supervised classification datasets (Sec.~\ref{sec:exp_openml}) %
and an open source fMRI simulator (Sec.~\ref{sec:bci_exp}).
We evaluate our approach by comparing: (1) CB: Contextual Bandits with exact reward, (2) \oldIGL: The method proposed by~\citet{xie2021interaction} which assumes the feedback vector contains no information about the context and reward, and (3) \newIGL: The proposed method which assumes that the feedback vector could contain information about the action but is conditionally independent of the context given the reward. Note that contextual bandits (CB) is a skyline compared to both \oldIGL and \newIGL, since it need not disambiguate the latent reward. All methods use logistic regression with a linear representation. At test time, each method takes the argmax of the policy. %
Following the practical instantiation of Assumption~\ref{asmp:pi_bad} (Section~\ref{sec:sym_brk}), we know if the dataset has a balanced action distribution (no action belongs to more than 50\% of the samples), selecting uniformly random policy $\pi_\unif$ as $\pi_\bad^a$ for all $a \in \Acal$ satisfies Assumption~\ref{asmp:pi_bad}. Therefore, in this section, our experiments are all based on the dataset with a balanced action distribution, and we select $\pi_\bad^a = \pi_\unif$ for all $a \in \Acal$.

\subsection{Large-scale Experiments with OpenML CC-18 Datasets}
\label{sec:exp_openml}
To verify that our proposed algorithm scales to a variety of tasks, we evaluate performance on more than 200 datasets from the publicly available OpenML Curated Classification Benchmarking Suite~\citep{vanschoren2015openml,casalicchio2019openml,feurer2021openml,bischl2021openml}. OpenML CC-18 datasets are licensed under CC-BY license\footnote{\url{https://creativecommons.org/licenses/by/4.0/}} and the platform and library are licensed under the BSD (3-Clause) license\footnote{\url{https://opensource.org/licenses/BSD-3-Clause}}.
At each time step, the context $x_t$ is generated uniformly at random. 
The learner selects an action $a_t \in \{0,\dots,K-1\}$ as the predicted label of $x_t$ (where $K$ is the total number of actions available in the environment). The binary reward $r_t$ is the correctness of the prediction label $a_t$. The feedback $y_t$ is a two dimensional vector $(a_t,r_t)$. Each dataset has a different sample size ($N$) and a different set of available actions ($K$). We sample datasets with 3 or more actions, and with a balanced action distribution (no action belongs to more than $50\%$ of the samples) to satisfy Assumption 4.
We use $90\%$ of the data for training and the remaining $10\%$ for evaluation. Additional details of setting up the experiment are in Appendix~\ref{sec:exp_detail}. The results are averaged over $20$ trials and shown in Table~\ref{tab:openml_action_exclusive}. 

\begin{table*}[!htb]
    \centering
    \resizebox{\textwidth}{!}{
    \begin{tabular}{|c|c|c|c||c|c||c|c|}
    \hline
         \begin{tabular}[x]{@{}c@{}}Dataset\\Criteria\end{tabular} & \begin{tabular}[x]{@{}c@{}}Dataset\\Count\end{tabular} & \begin{tabular}[x]{@{}c@{}}Constant\\Action\end{tabular} & \begin{tabular}[x]{@{}c@{}}CB Policy \\Accuracy (\%)\end{tabular} & \begin{tabular}[x]{@{}c@{}}IGL (full CI)\\Policy Accuracy (\%)\end{tabular} &  \begin{tabular}[x]{@{}c@{}}Performance\\w.r.t CB\end{tabular} & \begin{tabular}[x]{@{}c@{}}AI-IGL\\Policy Accuracy (\%)\end{tabular} & \begin{tabular}[x]{@{}c@{}}Performance\\w.r.t CB\end{tabular}\\
        \hline
        $K\geq3$ & $271$ & $25.28\pm2.96$ & $\mathbf{57.98\pm5.65}$ & $15.65\pm2.30$ & $0.30\pm0.04$ & $\mathbf{35.74\pm1.45}$ & $\mathbf{0.59\pm0.02}$\\
        \hline
        $K\geq$, $N\geq70000$ & 83 & $22.57\pm3.14$ & $\mathbf{58.41\pm5.04}$ & $11.91\pm1.80$ & $0.22\pm0.04$ & $\mathbf{50.11\pm2.98}$ & $\mathbf{0.79\pm0.03}$\\
        \hline
    \end{tabular}}
    \caption{Results in the OpenML environments with two-dimensional action-inclusive feedback. Average and standard error reported over 20 trials for each algorithm. `Performance w.r.t.~CB' reports the ratio of an IGL method's policy accuracy over CB policy accuracy. %
    }
    \vspace{-3mm}
    \label{tab:openml_action_exclusive}
\end{table*}

\subsection{Experiments on Realistic fMRI Simulation}
\label{sec:bci_exp}

We also evaluate our approach on a simulated BCI experiment, where a human imagines a number and the agent faces a number-decoding task (Fig.~\ref{fig:bci}). We simulate an fMRI based BCI setup, recording the fMRI BOLD (Blood Oxygenation Level Dependent) response in the parietal cortex of a human brain. We have selected the parietal cortex since both numerical processing~\citep{santens2010number} and match-mismatch processing are reported to involve this region~\citep{male2020quest}, and in principle it could feasibly be recorded by fNIRS (with similar spatial resolution to fMRI)~\citep{naseer2015fnirs}. The task requires an intelligent agent to decode the numbers the human imagines. This task is motivated by real-world examples where an implant could assist a person, user, or patient by communicating their mental states to the outside world, whether for therapeutic purposes or to control hardware or software. 

The task setup consists of the following steps---(1) Human Imagination: A human participant imagines a number from a set of digits (i.e., $7,8,9$), (2) Agent Decoding: an agent (Full CI IGL or AI-IGL) decodes the number being imagined using only the fMRI signals of the human participant, with no labels and no supervised pretraining, (3) Human Perception: the agent visually displays the number it decoded to the human, which the human brain perceives, and hence the fMRI signals would contain patterns related to the perception of the agent's guess, (4) Human Judgment: Once the human judges whether the agent-guessed number yields a match or a mismatch to the number they imagined, fMRI signals reflect the participant's judgment. This simple setup mimics many relevant applications of BCI with a realistic simulation of fMRI signals (generalizable to the more cost-efficient fNIRS). We used a third party simulation software, offering a strong test bed for the algorithms compared here.

We use an open source simulator~\citep{ellis2020facilitating} to simulate BCI-like, interactive fMRI signals for the imagination, perception, and judgment phases described above. This simulator leverages real human fMRI data~\citep{bejjanki2017noise} to simulate signals for different regions of the brain. Prior neuroscience studies have shown that the parietal cortices in the human brain (in particular the posterior parietal cortex) process numbers~\citep{santens2010number}, while the frontoparietal cortices  process visual attention as well as match-mismatch judgments~\citep{male2020quest}. These studies, together with the low temporal resolution of fMRI/fNIRS~\citep{glover2011overview}, warrant the assumption we make here: that number perception and judgments signals can both be decoded from the same region, i.e., the parietal cortex. Therefore, per AI-IGL's theoretical assumptions, the two signals can be mixed in the brain's implicit feedback signal, here simulated as an fMRI or fNIRS signal, in arbitrary ways. Thus, we selected a region of interest (ROI) in the parietal cortex to cover the mixed feedback of perceiving the numbers and making a match-mismatch judgments. This simulated BCI setting is precisely the sort of realistic condition that our proposed AI-IGL solution is designed to handle well. Moreover, it is quite possible to conduct this experiment in principle using fMRI, or more realistically using fNIRS–which is portable, noninvasive, easy to use, and can be focused on a specific patch on the surface of the brain. %
Each event (perception, imagination, judgment) evokes different activity in the human brain, within the same set of $4\times4\times4$ voxels, generating a 64-dimensional activity vector. We generate a time course of 666 trials, each of them including a sequence of events where an imagination event 
is followed by a response from the agent, then a perception event, and a judgment event. Thus, the 666 trials include 1998 simulated brain events total. Each event is 2 seconds long, with an inter-stimulus  (ISI) of 7 seconds, and a temporal resolution of 10 seconds. For each event, the simulator combines spatiotemporal noise (trial noise) and %
voxel activation patterns to simulate activity signals. %
We use the imagination signal as context (Fig.~\ref{fig:bci}, Context) and a weighted average of the perception and judgement signals to simulate the brain's mixed feedback signal (Fig.~\ref{fig:bci}, Feedback).

We evaluate the performance of both full CI IGL and AI-IGL on this simulated and interactive number-decoding task, under varying levels of feedback noise. The results are shown in Table~\ref{tab:bci_results}. We found that full CI IGL failed at the task, remaining at chance level guessing. However, the novel algorithm proposed here, AI-IGL, was capable of learning to dissociate the human judgment from a mixed feedback signal. %
The significance of these findings are two-fold. First, they offer a realistic approach to simulating BCI experiments in order to test theoretical advances before conducting real world experiments. Second, our results further validate the theoretical advances of AI-IGL compared to the previous IGL approach for a realistic interactive brain-computer interface.

\begin{table*}[!htb]
    \centering
    \begin{tabular}{|c|c|c|c|}
    \hline
         Algorithm & \begin{tabular}[x]{@{}c@{}} 1\% noise\end{tabular}    & \begin{tabular}[x]{@{}c@{}}5\% noise\end{tabular}  & \begin{tabular}[x]{@{}c@{}}10\% noise\end{tabular}  \\
        \hline
        \hline
        \oldIGL & $32.60\pm0.24$  & $33.75\pm0.29$ & $33.33 \pm 0.29$\\
        \hline
        \newIGL & $\mathbf{89.10 \pm 3.58}$ & $\mathbf{76.60 \pm 4.09}$ & $\mathbf{64.25 \pm 4.16}$\\
        \hline
    \end{tabular}
    \caption{%
    Policy accuracy for the simulated BCI experiments with action-inclusive feedback. Average and standard error reported over 20 trials for each algorithm.
    }
    \vspace{-6mm}
    \label{tab:bci_results}
\end{table*}

\section{Discussion}
\label{sec:discussion}
We have presented a new approach to solving Interaction-Grounded Learning, in which an agent learns to interact with the environment in the absence of any explicit reward signals. Compared to a prior solution~\citep{xie2021interaction}, the proposed AI-IGL approach %
removes the assumption of conditional independence of actions and the feedback vector %
by treating the latent reward as a latent state. It thereby provably solves IGL for action-inclusive feedback vectors. 
By viewing the feedback as containing an action-(latent) reward pair which is an unobserved latent space, we propose latent reward learning using a contrastive learning approach. 
This solution concept 
naturally connects to latent state discovery in rich-observation reinforcement learning~\citep[e.g.,][]{dann2018oracle,du2019provably,misra2020kinematic}. On the other hand, 
different from rich-observation RL, the problem of IGL also contains a unique challenge in identifying the semantic meaning of the decoded class, which is addressed by a symmetry-breaking procedure.  In this work, we focus on binary latent rewards, %
for which symmetry breaking is possible using one policy (per action).  Breaking symmetry in more general latent reward spaces is a topic for future work. 
A possible negative societal impact of this work can be performance instability, %
especially with inappropriate use of the techniques in risk-sensitive applications. 
Barring intentional misuse, we envision several potential benefits of the proposed approach. The proposed algorithm broadens the scope of IGL's feasibility for real-world applications. %
Imagine an agent being trained to interpret the brain signals of a user to control a prosthetic arm. %
The brain’s response (feedback vector) to an action is a neural signal that may contain information about the action itself. 
This is so prevalent in neuroimaging that 
fMRI studies routinely use specialized techniques to orthogonalize different information (e.g. action and reward) within the same signal~\citep{momennejad2012human,momennejad2013encoding,belghazi2018mine,shah2020hardness}.   %
Another example is a continually self-calibrating eye tracker, %
used by people with motor disabilities such as ALS~\citep{hansen2004gaze,liu2010eye,mott2017improving,gao2021x2t}. %
A learning agent adapting to the ability of such users %
can encounter feedback directly influenced by the calibration correction action.
The proposed approach is
a stepping stone on the path to solving IGL for complex interactive settings, overcoming the need for explicit rewards as well as explicit separation of action and feedback information.

\section*{Acknowledgments}
The authors would like to thank Mark Rucker for sharing methods and code for computing and comparing diagnostic features of OpenML datasets.
NJ acknowledges funding support from ARL Cooperative Agreement W911NF-17-2-0196, NSF IIS-2112471, NSF CAREER award, and Adobe Data Science Research Award.
\bibliographystyle{plainnat}
\bibliography{ref}

\clearpage
\allowdisplaybreaks

\appendix
\onecolumn

\begin{center}
{\LARGE Appendix}
\end{center}

\section{Detailed Proofs}
\label{sec:proofs}

\begin{proof}[\pfname{Proposition}{\ref{prop:decoding}}]
Over this proof, we follow the definition of \Eqref{eq:def_far_psiar} for each $(f,\psi) \in \Fcal \times \Psi$.
For any $\pi$, let $d^\pi_a \coloneqq \sum_{x}d_0(x) \pi(a|x)$ and $\rho^\pi_a \coloneqq \nicefrac{1}{d^\pi_a}\sum_{x}d_0(x) \pi(a|x) \Pr(r=1|x,a)$, then
\begin{align}
\E_{\mu_a}\left[f(x) \psi(y)\right] = &~ \frac{1}{d^\mu_a} \sum_{x,y} d_0(x) \mu(a|x) f(x) \Pr(y|a,r = 0)\psi(y) \Pr(r=0|x,a)
\\
&~ + \frac{1}{d^\mu_a} \sum_{x,y} d_0(x) \mu(a|x) f(x) \Pr(y|a,r = 1) \psi(y) \Pr(r=1|x,a)
\tag{by $x \indep y | r,a$}
\\
\overset{\text{(a)}}{=} &~ (1 - \rho^\mu_a ) f_{a,0} \psi_{a,0} + \rho^\mu_a f_{a,1} \psi_{a,1},
\end{align}
and
\begin{align}
\E_{\mu_a}\left[ f(x) \right] \E_{\mu_a}\left[ \psi(y) \right] = &~ \left( \frac{1}{d^\mu_a} \sum_{x} d_0(x) \mu(a|x) f(x) \right) \left( \frac{1}{d^\mu_a} \sum_{x,y} d_0(x) \mu(a|x) \Pr(y|a) \psi(y) \right)
\\
\overset{\text{(b)}}{=} &~ \left( (1 - \rho^\mu_a) f_{a,0} + \rho^\mu_a f_{a,1} \right) \left( (1 - \rho^\mu_a) \psi_{a,0} + \rho^\mu_a \psi_{a,1} \right)
\\
= &~ (1 - \rho^\mu_a)^2 f_{a,0} \psi_{a,0} + \left( \rho^\mu_a \right)^2 f_{a,1} \psi_{a,1} + (1 - \rho^\mu_a) \rho^\mu_a (f_{a,0} \psi_{a,1} + f_{a,1} \psi_{a,0}).
\end{align}
To see why (a) and (b) hold, we use the following argument. By definition of $f_{a,r}$ and $\psi_{a,r}$, we have
\begin{align}
f_{a,0} = &~ \sum_x \Pr(x|a,r=0) f_a(x,a)
\\
= &~ \sum_x \frac{\Pr(x,a,r=0) f_a(x,a)}{\Pr(a,r=0)}
\\
= &~ \frac{\sum_x d_0(x) \mu(a|x) f_a(x,a) \Pr(r=0|x,a)}{\sum_x d_0(x) \mu(a|x) \Pr(r=0|x,a)}.
\\
\Longrightarrow d^\mu_a (1 - \rho^\mu_a) f_{a,0} = &~ \sum_{x} d_0(x) \mu(a|x) f_a(x,a) \Pr(r=0|x,a).
\end{align}
Then, we have
\begin{align}
&~ \frac{1}{d^\mu_a} \sum_{x,y} d_0(x) \mu(a|x) f_a(x,a) \Pr(y|a,r = 0)\psi_a(y,a) \Pr(r=0|x,a)
\\
= &~ \frac{1}{d^\mu_a} \sum_{x} d_0(x) \mu(a|x) f_a(x,a) \Pr(r=0|x,a) \sum_{y} \Pr(y|a,r = 0)\psi_a(y,a)
\\
= &~ \frac{\psi_{a,0}}{d^\mu_a} \sum_{x} d_0(x) \mu(a|x) f_a(x,a) \Pr(r=0|x,a) 
\\
= &~ (1 - \rho^\mu_a ) f_{a,0} \psi_{a,0}.
\end{align}
Similar procedure also induces the remaining terms of (a) and (b).

Therefore, combining the two equalities above, we obtain
\begin{align}
&~ \E_{\mu_a}\left[f(x) \psi(y)\right] - \E_{\mu_a}\left[ f(x) \right] \E_{\mu_a}\left[ \psi(y) \right]
\\
= &~ (1 - \rho^\mu_a ) \rho^\mu_a (f_{a,0} \psi_{a,0} + f_{a,1} \psi_{a,1}) - (1 - \rho^\mu_a) \rho^\mu_a (f_{a,0} \psi_{a,1} + f_{a,1} \psi_{a,0})
\\
= &~ (1 - \rho^\mu_a ) \rho^\mu_a (f_{a,0} \psi_{a,0} + f_{a,1} \psi_{a,1} - f_{a,0} \psi_{a,1} - f_{a,1} \psi_{a,0})
\\
\label{eq:dev_L_a}
= &~ (1 - \rho^\mu_a ) \rho^\mu_a (f_{a,1} - f_{a,0}) (\psi_{a,1} - \psi_{a,0}).
\end{align}
This completes the proof.
\end{proof}

\begin{proof}[\bf\em Lower bound of $\Delta_\Fcal$ when $\mu = \pi_\unif$]
For any $f \in \Fcal$, let $V_f(\pi) \coloneqq \sum_x \E_{d_0 \times \pi}[f(x,a)] \in [0,1]$. Then,
\begin{align}
&~ \E_\mu\left[ f(x,\abar) | \abar, r = 1 \right] - \E_\mu\left[ f(x,\abar) | \abar, r = 0 \right]
\\
= &~ \frac{\sum_x \Pr(x,\abar,r=1) f(x,\abar)}{\sum_x \Pr(x,\abar,r=1)} - \frac{\sum_x \Pr(x,\abar,r=0) f(x,\abar)}{\sum_x \Pr(x,\abar,r=0)}
\\
= &~ \frac{\sum_x d_0(x) \mu(\abar|x) R(x,\abar) f(x,\abar)}{\sum_x d_0(x) \mu(\abar|x) R(x,\abar)} - \frac{\sum_x d_0(x) \mu(\abar|x) f(x,\abar) (1 - R(x,\abar))}{\sum_x d_0(x) \mu(\abar|x) (1 - R(x,\abar))}
\\
= &~ \frac{\sum_x d_0(x) R(x,\abar) f(x,\abar)}{\sum_x d_0(x) R(x,\abar)} - \frac{\sum_x d_0(x) f(x,\abar) (1 - R(x,\abar))}{\sum_x d_0(x) (1 - R(x,\abar))}
\tag{$\mu = \pi_\unif$}
\\
= &~ \frac{\sum_x d_0(x) f(x,\abar)  R(x,\abar)}{V(\pi_\abar)} - \frac{\sum_x d_0(x) f(x,\abar) (1 - R(x,\abar))}{1 - V(\pi_\abar)}
\tag{$\pi_\abar$ denotes the constant policy with action $\abar$}
\\
= &~ \frac{\sum_x d_0(x) f(x,\abar)  R(x,\abar) - V_f(\pi_\abar)V(\pi_\abar)}{V(\pi_\abar)}
\\
&~ - \frac{\sum_x d_0(x) f(x,\abar) (1 - R(x,\abar)) - V_f(\pi_\abar)(1 - V(\pi_\abar))}{1 - V(\pi_\abar)}
\\
= &~ \frac{\Cov_{\pi_\abar}(f,R)}{V(\pi_\abar)} - \frac{\Cov_{\pi_\abar}(f,1-R)}{1 - V(\pi_\abar)}
\\
\overset{\text{(a)}}{=} &~ \frac{\Cov_{\pi_\abar}(f,R)}{V(\pi_\abar)} + \frac{\Cov_{\pi_\abar}(f,R)}{1 - V(\pi_\abar)}
\\
= &~ \frac{\Cov_{\pi_\abar}(f,R)}{V(\pi_\abar)(1 - V(\pi_\abar))}
\\
\Longrightarrow &~ \left|\E_\mu\left[ f(x,\abar) | \abar, r = 1 \right] - \E_\mu\left[ f(x,\abar) \right| \abar, r = 0 \right]| \geq 4 |\Cov_{\pi_\abar}(f,R)|,
\end{align}
where (a) follows from
\begin{align}
\Cov_{\pi_\abar}(f,1-R) = &~ \E_{\pi_\abar}[(f(x,a) - V_f(\pi_\abar))(1 - R(x,a) - 1 - V(\pi_\abar))]
\\
= &~ -\E_{\pi_\abar}[(f(x,a) - V_f(\pi_\abar))(R(x,a) - V(\pi_\abar))]
\\
= &~ -\Cov_{\pi_\abar}(f,R).
\end{align}
This completes the proof.
\end{proof}

\begin{proof}[\pfname{Lemma}{\ref{lem:ground_reward}}]
Let,
\begin{align}
(\ftilde_{\abar},\psitilde_{\abar}) \leftarrow \argmax_{(f,\psi) \in \Fcal \times \Psi} \Jcal_{\abar,\mu}(f,\psi).
\end{align}

Over this proof, we define
\begin{align}
\Delta_\Fcal = \max_{f \in \Fcal} (f_{a,1} - f_{a,0}).
\end{align}

Now, for $\Jcal_{\abar,\Dcal}(f,\psi)$ with any $(f,\psi)$, we also have with probability at least $1 - \delta$
\begin{align}
\left| \Jcal_{\abar,\Dcal}(f,\psi) - \Jcal_{\abar,\mu}(f,\psi) \right| \leq \varepsilon_{\stat,\abar}.
\end{align}
This means
\begin{align}
&~ \Jcal_{\abar,\Dcal}(\ftilde_{\abar},\psitilde_{\abar}) \leq  \Jcal_{\abar,\Dcal}(f_{\abar},\psi_{\abar})
\\
\Longrightarrow &~ \Jcal_{\abar,\mu}(\ftilde_{\abar},\psitilde_{\abar}) - 2 \varepsilon_{\stat,\abar} \leq \Jcal_{\abar,\mu}(f_{\abar},\psi_{\abar})
\\
\Longrightarrow  &~ \underbrace{(1 - \rho^\mu_a) \rho^\mu_a}_{=\nicefrac{K - 1}{K^2},~\text{as $\mu = \pi_\unif$}} \cdot (\ftilde_{\abar,1} - \ftilde_{\abar,0}) (\psitilde_{\abar,1} - \psitilde_{\abar,0}) - 2 \varepsilon_{\stat,\abar} \leq (1 - \rho^\mu_a) \rho^\mu_a (f_{\abar,1} - f_{\abar,0}) (\psi_{\abar,1} - \psi_{\abar,0}).
\tag{by \Eqref{eq:dev_L_a}}
\\
\Longrightarrow &~ (f_{\abar,1} - f_{\abar,0}) (\psi_{\abar,1} - \psi_{\abar,0}) \geq (\ftilde_{\abar,1} - \ftilde_{\abar,0}) (\psitilde_{\abar,1} - \psitilde_{\abar,0}) - \frac{2 K^2}{(K - 1)} \varepsilon_{\stat,\abar}.
\end{align}
This implies
\begin{align}
\label{eq:f_guarantee}
|f_{\abar,1} - f_{\abar,0}| \geq  &~ \Delta_\Fcal - \frac{2 K^2}{(K - 1)} \varepsilon_{\stat,\abar}
\\
|\psi_{\abar,1} - \psi_{\abar,0}| \geq &~ 1 - \frac{2 K^2}{\Delta_\Fcal (K - 1)} \varepsilon_{\stat,\abar}
\\
\label{eq:psi_guarantee}
\Longrightarrow |\psi^\star_{\abar,1} - \max\{\psi_{\abar,1}, \psi_{\abar,0}\}|, |\min\{\psi_{\abar,1}, \psi_{\abar,0}\} - \psi^\star_{\abar,0}| \leq &~ \frac{2 K^2}{\Delta_\Fcal (K - 1)} \varepsilon_{\stat,\abar}.
\end{align}
This completes the proof.
\end{proof}

\begin{proof}[\pfname{Lemma}{\ref{lem:estimation_per_action}}]
We now provide the proof for any fixed $\abar \in \Acal$.
We define $d^\pi_\abar \coloneqq \sum_{x}d_0(x) \pi(\abar|x)$, and $R_\psi(x,a) = \E[\psi(y,a)|x,a]$ as the decoded reward by $\psi$ for all $(x,a,\psi) \in \Xcal \times \Acal \times \Psi$.
Also, let $\rho^{\pi_\bad^{\abar}}_{\abar}$, $\rhohat^{\pi_\bad^{\abar}}_{\abar}$, $\varrho^{\pi_\bad^{\abar}}_{\abar}$, $\varrhohat^{\pi_\bad^{\abar}}_{\abar}$ be,
\begin{align}
\label{eq:def_rho}
&~ \rho^{\pi_\bad^{\abar}}_{\abar} = \frac{\overbrace{\textstyle\sum_{x}d_0(x) \pi_\bad^{\abar}(\abar|x) R(x,\abar)}^{\eqqcolon\text{(I).pop}}}{\underbrace{\textstyle\sum_{x}d_0(x) \pi_\bad^{\abar}(\abar|x)}_{\eqqcolon\text{(II).pop}}},  \qquad  \rhohat^{\pi_\bad^{\abar}}_{\abar} = \frac{\overbrace{\textstyle\sum_{(x,a,y) \sim \Dcal} \pi_\bad^{\abar}(a|x) R(x,a) \1(a = \abar)}^{\eqqcolon\text{$\sum_{(x,a,y) \sim \Dcal} \1(a = \abar) \cdot $ (I).emp}}}{\underbrace{\textstyle\sum_{(x,a,y) \sim \Dcal} \pi_\bad^{\abar}(a|x) \1(a = \abar)}_{\eqqcolon\text{$\sum_{(x,a,y) \sim \Dcal} \1(a = \abar) \cdot $(II).emp}}}
\\
\label{eq:def_varrho}
&~ \varrho^{\pi_\bad}_{\abar} = \frac{\overbrace{\textstyle\sum_{x}d_0(x) \pi_\bad^{\abar}(\abar|x) R_{\psi_\abar}(x,\abar)}^{\coloneqq\text{(III).pop}}}{\sum_{x} d_0(x) \pi_\bad^{\abar}(\abar|x)}, \qquad \varrhohat^{\pi_\bad^{\abar}}_{\abar} = \frac{\overbrace{\textstyle\sum_{(x,a,y) \sim \Dcal}\pi_\bad^{\abar}(a|x) \psi_a(y,a) \1(a = \abar)}^{\eqqcolon\text{$\sum_{(x,a,y) \sim \Dcal} \1(a = \abar) \cdot $ (III).emp}}}{\sum_{(x,a,y) \sim \Dcal}\pi_\bad^{\abar}(a|x) \1(a = \abar)}
\end{align}
Over this section, we define $\varepsilon_{\stat,\abar}$ as,
\begin{align}
\varepsilon_{\stat,\abar} \coloneqq \Ocal\left( \sqrt{\frac{K d_{\Fcal, \Psi}}{|\Dcal|}} \right) \geq \Ocal\left(\sqrt{\frac{d_{\Fcal, \Psi}}{\sum_{(x,a,y) \sim \Dcal} \1(a = \abar)}}\right).
\tag{by~\citep[Lemma A.1]{xie2021policy}}
\end{align}

Then, we know with probability at least $1 - \delta$,
\begin{align}
\label{eq:rho_concentration}
\left| \text{(I).pop} - \text{(I).emp}\right|, \left| \text{(II).pop} - \text{(II).emp}\right|, \left| \text{(III).pop} - \text{(III).emp}\right| \leq \varepsilon_{\stat,\abar}
\end{align}
by a standard concentration argument. 

We now show when line~\ref{step:sym_brk_1}-\ref{step:sym_brk_2} in Algorithm~\ref{alg:IGL_x_CI} correctly break the symmetry. By Assumption~\ref{asmp:pi_bad}, we have $\rho^{\pi_\bad^{\abar}}_{\abar} \leq \frac{1}{2} - \eta$. In addition, by \Eqref{eq:rho_concentration},
\begin{align}
\label{eq:rho_concentration_final}
&~ \left| \rhohat^{\pi_\bad^{\abar}}_{\abar} - \rho^{\pi_\bad^{\abar}}_{\abar} \right|, \left| \varrhohat^{\pi_\bad^{\abar}}_{\abar} - \varrho^{\pi_\bad^{\abar}}_{\abar} \right| \geq 2 \varepsilon_{\stat,\abar} + \frac{2 \varepsilon_{\stat,\abar}^2}{c_m - \varepsilon_{\stat,\abar}}.
\end{align}

Recall the definition of $\rho^{\pi_\bad^{\abar}}_{\abar}$ and $\varrho^{\pi_\bad^{\abar}}_{\abar}$ (\Eqref{eq:def_rho} and \Eqref{eq:def_varrho}), we have
\begin{align}
\rho^{\pi_\bad^{\abar}}_{\abar} - \varrho^{\pi_\bad^{\abar}}_{\abar} = \frac{\sum_{x}d_0(x) \pi_\bad^{\abar}(\abar|x) \E[\psi^\star(y,\abar) - \psi_\abar(y,\abar)|x,\abar]}{\sum_{x}d_0(x) \pi_\bad^{\abar}(\abar|x) R_{\psi_\abar}(x,\abar)}.
\end{align}
According to \Eqref{eq:psi_guarantee}, if $\psi_{\abar,1} > \psi_{\abar,0}$, we have
\begin{align}
\varrho^{\pi_\bad^{\abar}}_{\abar} \leq &~ \rho^{\pi_\bad^{\abar}}_{\abar} + \frac{2 K^2}{(K - 1)} \varepsilon_{\stat,\abar}
\\
\Longrightarrow \varrhohat^{\pi_\bad^{\abar}}_{\abar} \leq &~ \rho^{\pi_\bad^{\abar}}_{\abar} + 2 \varepsilon_{\stat,\abar} + \frac{2 \varepsilon_{\stat,\abar}^2}{c_m - \varepsilon_{\stat,\abar}},
\tag{by \Eqref{eq:rho_concentration_final}}
\end{align}
otherwise,
\begin{align}
&~ \rho^{\pi_\bad^{\abar}}_{\abar} - \varrho^{\pi_\bad^{\abar}}_{\abar} = 1 - \frac{\sum_{x}d_0(x) \pi_\bad(\abar|x) \E[1 - \psi_\abar(y,\abar) - \psi^\star(y,\abar) |x,\abar]}{\sum_{x}d_0(x) \pi_\bad(\abar|x) R_{\psi_\abar}(x,\abar)}
\\
\Longrightarrow &~ \varrho^{\pi_\bad}_{\abar} \geq 1 - \rho^{\pi_\bad}_{\abar} - \frac{2 K^2}{\Delta_\Fcal (K - 1)} \varepsilon_{\stat,\abar}
\\
\Longrightarrow &~ \varrhohat^{\pi_\bad^{\abar}}_{\abar} \geq 1 - \rho^{\pi_\bad^{\abar}}_{\abar} - \frac{2 K^2}{\Delta_\Fcal (K - 1)} \varepsilon_{\stat,\abar} - 2 \varepsilon_{\stat,\abar} - \frac{2 \varepsilon_{\stat,\abar}^2}{c_m - \varepsilon_{\stat,\abar}}.
\tag{by \Eqref{eq:rho_concentration_final}}
\end{align}

To guarantee the correctness of the symmetry breaking step (line~\ref{step:sym_brk_1}-\ref{step:sym_brk_2} in Algorithm~\ref{alg:IGL_x_CI}), we need
\begin{enumerate}
    \item If $f_{\abar,1} > f_{\abar,0}$, $\varrhohat^{\pi_\bad^{\abar}}_{\abar} \leq \frac{1}{2}$,
    \item Otherwise, $\varrhohat^{\pi_\bad^{\abar}}_{\abar} > \frac{1}{2}$.
\end{enumerate}
That requires,
\begin{align}
\frac{2 K^2}{\Delta_\Fcal (K - 1)} \varepsilon_{\stat,\abar} + 2 \varepsilon_{\stat,\abar} + \frac{2 \varepsilon_{\stat,\abar}^2}{c_m - \varepsilon_{\stat,\abar}} \leq \eta.
\end{align}
This can be induced by
\begin{align}
\varepsilon_{\stat,\abar} \leq \Ocal\left( \min\left\{\nicefrac{\eta \Delta_\Fcal}{K}, c_m \right\} \right)
\quad \left( \Longleftarrow \quad
|\Dcal| \geq \Ocal \left( \frac{K^3 d_{\Fcal, \Psi}}{\left(\min\{\eta \Delta_\Fcal, K c_m \}\right)^2} \right) \right). \tag*{\qedhere}
\end{align}
\end{proof}

\begin{proof}[\pfname{Theorem}{\ref{thm:main_thm}}]
Combining Lemma~\ref{lem:ground_reward} and Lemma~\ref{lem:estimation_per_action}, we know for any $\abar \in \Acal$
\begin{align}
|\psi^\star_{\abar,1} - \psi_{\abar,1}|, |\psi_{\abar,0} - \psi^\star_{\abar,0}| \leq &~ \frac{2 K^2}{\Delta_\Fcal (K - 1)} \varepsilon_{\stat,\abar}.
\end{align}
Then,
\begin{align}
&~ V(\pi^\star) - V(\pihat)
\\
= &~ \E_{d_0 \times \pi^\star}[r] - \E_{d_0 \times \pihat}[r]
\\
= &~ \E_{(x,a,y) \sim d_0 \times \pi^\star}[\psi^\star(y,a)] - \E_{(x,a,y) \sim d_0 \times \pihat}[\psi^\star(y,a)]
\tag{by Assumption~\ref{asmp:realizability}}
\\
= &~ \E_{(x,a,y) \sim d_0 \times \pi^\star}[\psi'_a(y,a)] - \E_{(x,a,y) \sim d_0 \times \pihat}[\psi'_a(y,a)]
\\
&~ + \left| \E_{(x,a,y) \sim d_0 \times \pi^\star}[\psi^\star(y,a) - \psi'_a(y,a)] \right| + \left| \E_{(x,a,y) \sim d_0 \times \pihat}[\psi^\star(y,a) - \psi'_a(y,a)] \right|
\\
\leq &~ \E_{(x,a,y) \sim d_0 \times \pi^\star}[\psi'_a(y,a)] - \E_{(x,a,y) \sim d_0 \times \pihat}[\psi'_a(y,a)]
\\
&~ + (V(\pi^\star) + V(\pihat)) \max_{a \in \Acal} |\psi^\star_{\abar,1} - \psi_{\abar,1}| + (2 - V(\pi^\star) - V(\pihat)) \max_{a \in \Acal} |\psi_{\abar,0} - \psi^\star_{\abar,1}|
\\
\leq &~ \E_{(x,a,y) \sim d_0 \times \pi^\star}[\psi'_a(y,a)] - \E_{(x,a,y) \sim d_0 \times \pi^\star}[\psi'_a(y,a)] + \frac{4 K^2}{\Delta_\Fcal (K - 1)} \varepsilon_{\stat,\abar}
\\
\leq &~ \varepsilon_\CB + \frac{4 K^3}{\Delta_\Fcal (K - 1)} \varepsilon_{\stat,\abar}.
\tag{by the property of $\CB$ oracle in Definition~\ref{def:CB_oracle}}
\end{align}
This completes the proof.
\end{proof}

\section{Experiments on the MNIST dataset}
\label{sec:exp_mnist}

\begin{figure*}[t!]
\begin{subfigure}{0.32\textwidth}
\centering
\includegraphics[width=0.7\textwidth]{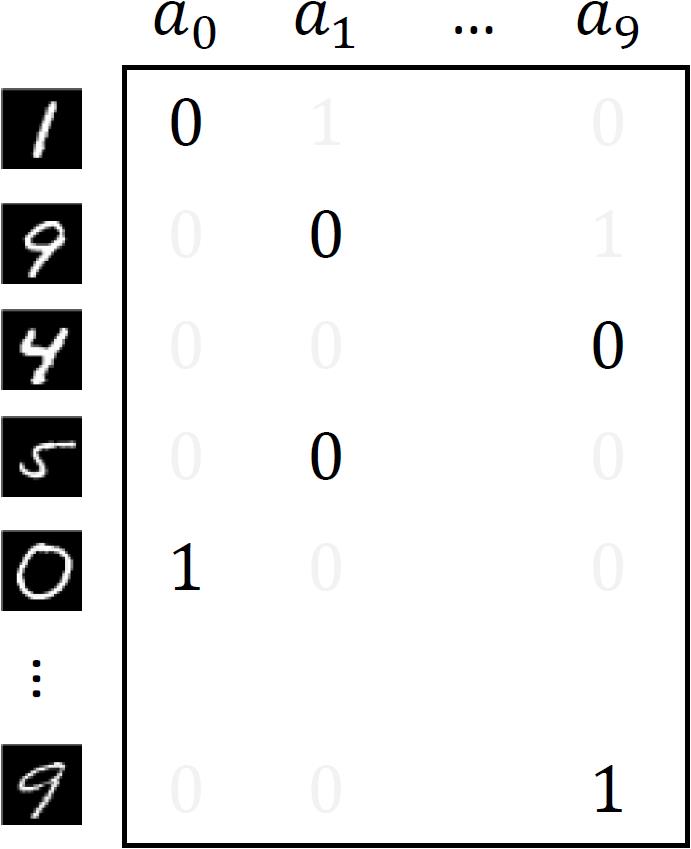}
\caption{Contextual Bandits (CB)}
\label{fig:CB_mnist}
\end{subfigure}
~
\begin{subfigure}{0.32\textwidth}
\centering
\includegraphics[width=.7\textwidth]{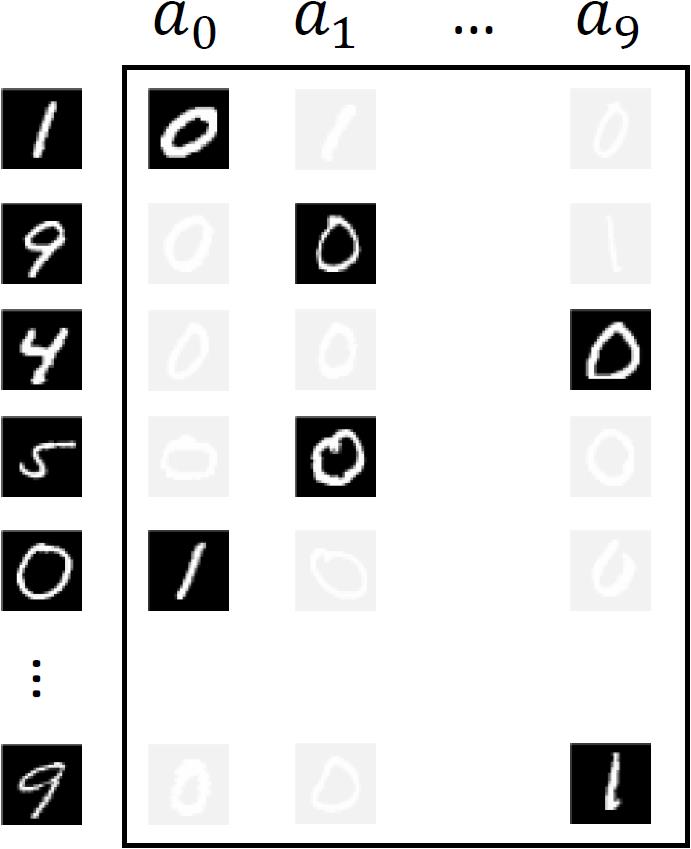}
\caption{\oldIGL~\citep{xie2021interaction}}
\label{fig:old_idl_mnist}
\end{subfigure}
~
\begin{subfigure}{0.32\textwidth}
\centering
\includegraphics[width=.7\textwidth]{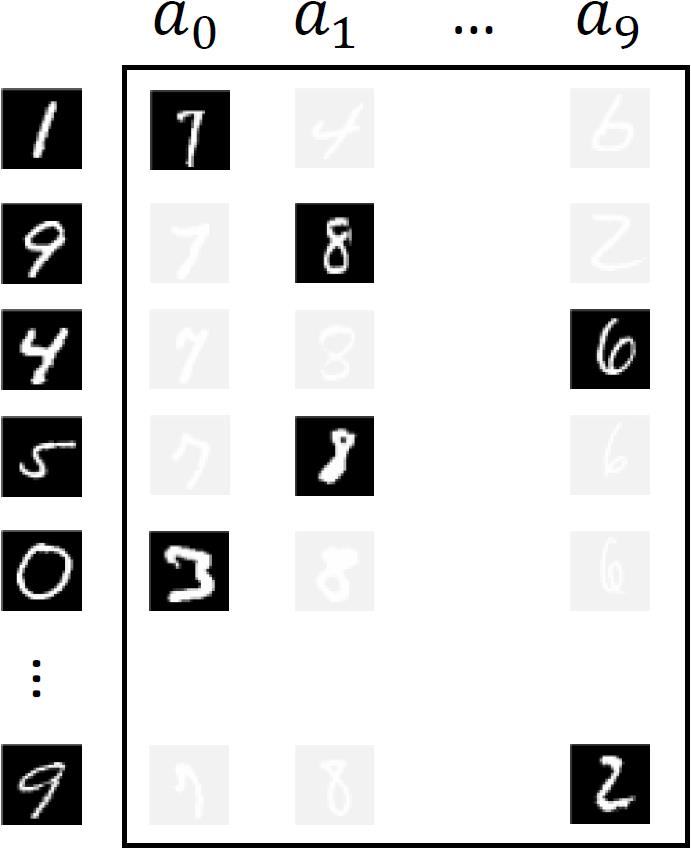}
\caption{\newIGL}
\label{fig:new_idl_mnist}
\end{subfigure}
\caption{Different learning approaches based on the MNIST dataset. The gray number/image denotes the unobserved reward/feedback vector. {\bf Figure~\ref{fig:CB_mnist}:} In the contextual bandits setting, the exact reward information on the selected action can be observed. {\bf Figure~\ref{fig:old_idl_mnist}:} In \oldIGL, the feedback vector is generated only based on the latent reward. {\bf Figure~\ref{fig:new_idl_mnist}:} In \newIGL, the feedback vector can be generated based on both latent reward and selected action.}
\label{fig:mnist_setting}
\end{figure*}

The environment for this experiment is generated from the supervised classification MNIST dataset~\citep{lecun1998gradient} which is licensed under Attribution-Share Alike 3.0 license\footnote{\url{https://creativecommons.org/licenses/by-sa/3.0/}}. At each time step, the context $x_t$ is generated uniformly at random. Then the learner selects an action $a_t \in \{0,\dots,9\}$ as the predicted label of $x_t$. The binary reward $r$ is the correctness of the prediction label $a_t$. The high-dimensional feedback vector $y_t$ is an image of the digit $(a_t + 6r - 3) \mod 10$. An example is shown in Figure~\ref{fig:mnist_setting}. Our results are averaged over 20 trials.

\begin{table*}[!htb]
    \centering
    \begin{tabular}{|c|c|c|}
    \hline
         Algorithm & \begin{tabular}[x]{@{}c@{}}Policy accuracy for\\action-inclusive feedback (\%)\end{tabular}    & \begin{tabular}[x]{@{}c@{}}Policy accuracy for\\action-exclusive feedback (\%)\end{tabular}  \\
        \hline
        CB & $\mathbf{87.64 \pm 0.25}$ & $\mathbf{87.64 \pm 0.25}$ \\
        \hline
        \oldIGL & $9.18 \pm 2.10$ & $86.13 \pm 0.52$ \\
        \hline
        \newIGL & $\mathbf{83.63 \pm 1.25}$ & $\mathbf{86.14 \pm 0.94}$\\
        \hline
    \end{tabular}
    \caption{Results in the MNIST environment with high-dimensional action-inclusive and action-exclusive feedback. Average and standard error reported over 20 trials for each algorithm.
    }
    \label{tab:mnist_action_inclusive}
\end{table*}

To highlight that the proposed algorithm still operates well under conditions where the feedback vector does not include action information, we also perform experiments under the setting introduced by~\citet{xie2021interaction} on the MNIST environment. This setting is similar to the one described in Section~\ref{sec:exp_mnist} except the feedback vector is the image of the digit $r$ instead of $(a_t + 6r - 3) \mod 10$. We find that under this action-exclusive feedback setting, the new proposed algorithm \newIGL works as well as the \oldIGL. This signifies that our algorithm, which incorporates the presence of action information in the feedback vector, does not hurt performance in cases when the action information is missing from the feedback vector.

\section{Additional Experimental Details}
\label{sec:exp_detail}
This section provides additional details on our implementation. The experiments on the MNIST dataset were conducted using a Google Colab, which was based on Intel Xeon CPU (2.30GHz) and 12 GB memory. Large scale experiments on the OpenML datasets were conducted on CPU instances of an internal cluster. No GPU was used. The prototype codes were built over Python, PyTorch\footnote{\url{https://pytorch.org/}}, and Vowpal Wabbit\footnote{\url{https://vowpalwabbit.org/}}. With the single process in the setup above, a single trial of each experiment took less than 30 minutes to finish. Each experiment runs for 10 epochs over the dataset it is being evaluated on. The data is shuffled only once before the training begins so each algorithms views the data in the same order in each epoch. %

The experiment is established based on the Brain Imaging Analysis Kit (brainiak)\footnote{\url{https://github.com/brainiak/brainiak}}, which provides the realistic simulation of functional Magnetic Resonance Imaging (fMRI) based on the real human data~\citep{bejjanki2017noise}. Our detailed setup, including the choice of the most of experiment parameters, follows from the demo provided in brainiak\footnote{\url{https://github.com/brainiak/brainiak/blob/master/examples/utils/fmrisim_multivariate_example.ipynb}}.

\subsection{Results for Individual OpenML datasets}

For each individual OpenML dataset, we report the mean performance and standard errors for the best contact action policy, contextual bandit (CB), IGL assuming conditional independence of feedback on context and action given the latent reward (full CI)~\citep{xie2021interaction}, and IGL with conditional independence of feedback on context given reward (AI-IGL) in Table~\ref{tab:all_openml}. The best constant action policy is computed based on the highest percentage of samples in the entire dataset with the same action. We do not report the standard error for it.
{\scriptsize \sf
\begin{longtable}{c|cccccc}
\caption{Average accuracy and standard errors for different algorithms on individual OpenML datasets.}
\\
\hline
Dataset ID & N & K& \begin{tabular}[x]{@{}c@{}}Best Constant\\Action Policy (\%)\end{tabular} & CB (\%)& IGL (full CI) (\%)& \newIGL (\%)\\
\hline
\endhead
\hline
\endfoot
\label{tab:all_openml}
6     & 20000   & 26   & 4.07  & 53.90$\pm$0.53 & 3.54$\pm$0.18  & 11.40$\pm$1.45   \\
7     & 226     & 24   & 25.22 & 34.35$\pm$2.70 & 5.65$\pm$1.69  & 8.26$\pm$2.26    \\
9     & 205     & 6    & 32.68 & 42.38$\pm$2.56 & 16.90$\pm$2.98 & 20.95$\pm$3.25   \\
11    & 625     & 3    & 46.08 & 85.79$\pm$1.07 & 32.38$\pm$4.33 & 25.00$\pm$6.60   \\
12    & 2000    & 10   & 10.00 & 87.10$\pm$0.79 & 10.17$\pm$1.04 & 32.17$\pm$4.25   \\
14    & 2000    & 10   & 10.00 & 67.75$\pm$0.77 & 10.98$\pm$0.68 & 20.47$\pm$3.15   \\
16    & 2000    & 10   & 10.00 & 87.37$\pm$0.42 & 9.85$\pm$0.73  & 48.52$\pm$4.34   \\
18    & 2000    & 10   & 10.00 & 59.13$\pm$0.94 & 9.25$\pm$0.71  & 26.48$\pm$3.48   \\
20    & 2000    & 10   & 10.00 & 92.25$\pm$0.49 & 9.82$\pm$0.75  & 43.35$\pm$4.35   \\
22    & 2000    & 10   & 10.00 & 71.55$\pm$0.88 & 9.88$\pm$0.81  & 23.85$\pm$3.64   \\
23    & 1473    & 3    & 42.70  & 48.85$\pm$0.92 & 38.61$\pm$1.72 & 31.79$\pm$1.66   \\
26    & 12960   & 5    & 33.33 & 89.86$\pm$0.23 & 19.66$\pm$3.31 & 77.34$\pm$2.83   \\
28    & 5620    & 10   & 10.18 & 93.99$\pm$0.20 & 10.30$\pm$0.75 & 66.49$\pm$4.47   \\
32    & 10992   & 10   & 10.41 & 88.10$\pm$0.22 & 9.62$\pm$0.59  & 74.47$\pm$3.01   \\
35    & 366     & 6    & 30.60 & 95.00$\pm$0.82 & 12.16$\pm$1.62 & 46.62$\pm$6.06   \\
36    & 2310    & 7    & 14.29 & 87.27$\pm$0.64 & 13.35$\pm$1.00 & 63.35$\pm$4.70   \\
39    & 336     & 8    & 42.56 & 71.76$\pm$2.20 & 11.47$\pm$2.69 & 23.53$\pm$4.86   \\
41    & 214     & 6    & 35.51 & 49.55$\pm$2.87 & 14.55$\pm$2.96 & 22.05$\pm$3.68   \\
42    & 683     & 19   & 13.47 & 60.22$\pm$1.89 & 4.57$\pm$0.70  & 20.36$\pm$2.46   \\
48    & 151     & 3    & 34.44 & 42.19$\pm$3.03 & 27.19$\pm$2.09 & 31.25$\pm$2.30   \\
54    & 846     & 4    & 25.77 & 50.12$\pm$1.56 & 23.88$\pm$0.91 & 28.94$\pm$2.19   \\
60    & 5000    & 3    & 33.84 & 86.02$\pm$0.27 & 33.85$\pm$0.48 & 68.82$\pm$4.89   \\
61    & 150     & 3    & 33.33 & 89.00$\pm$1.90 & 33.33$\pm$2.87 & 37.67$\pm$6.67   \\
62    & 101     & 7    & 40.59 & 75.00$\pm$4.40 & 10.91$\pm$3.99 & 23.18$\pm$5.32   \\
74    & 1000000 & 26   & 4.08  & 41.96$\pm$0.30 & 3.87$\pm$0.59  & 33.08$\pm$4.14   \\
75    & 1000000 & 7    & 32.33 & 60.04$\pm$0.09 & 15.91$\pm$3.47 & 58.35$\pm$0.28   \\
78    & 1000000 & 10   & 10.06 & 83.61$\pm$0.05 & 10.07$\pm$0.12 & 80.28$\pm$0.07   \\
115   & 1000000 & 10   & 10.04 & 93.85$\pm$0.03 & 9.81$\pm$0.99  & 92.96$\pm$0.05   \\
116   & 1000000 & 6    & 42.27 & 67.41$\pm$0.06 & 16.20$\pm$1.90 & 66.11$\pm$0.06   \\
117   & 1000000 & 6    & 42.27 & 66.79$\pm$0.08 & 19.84$\pm$5.02 & 65.11$\pm$0.10   \\
118   & 1000000 & 10   & 10.04 & 76.10$\pm$0.07 & 9.58$\pm$0.82  & 72.91$\pm$0.02   \\
119   & 55296   & 3    & 42.78 & 48.76$\pm$0.17 & 33.59$\pm$1.75 & 40.31$\pm$1.91   \\
123   & 1000000 & 10   & 10.17 & 93.02$\pm$0.04 & 11.90$\pm$1.77 & 90.73$\pm$0.02   \\
127   & 1000000 & 10   & 10.46 & 80.56$\pm$0.06 & 9.48$\pm$2.12  & 75.76$\pm$0.14   \\
129   & 1000000 & 6    & 30.46 & 97.47$\pm$0.01 & 11.11$\pm$2.06 & 96.54$\pm$0.01   \\
130   & 1000000 & 7    & 14.36 & 80.83$\pm$0.08 & 14.63$\pm$1.01 & 75.56$\pm$0.27   \\
133   & 137781  & 7    & 35.04 & 59.08$\pm$0.14 & 15.86$\pm$2.34 & 47.74$\pm$3.06   \\
134   & 1000000 & 19   & 13.33 & 88.33$\pm$0.05 & 5.64$\pm$1.05  & 83.17$\pm$0.04   \\
141   & 1000000 & 4    & 25.81 & 58.92$\pm$0.03 & 22.28$\pm$1.56 & 58.59$\pm$0.17   \\
147   & 1000000 & 3    & 33.78 & 83.32$\pm$0.10 & 33.22$\pm$0.14 & 82.46$\pm$0.07   \\
148   & 1000000 & 7    & 39.62 & 93.46$\pm$0.05 & 16.52$\pm$2.19 & 92.25$\pm$0.01   \\
149   & 1455525 & 10   & 44.97 & 61.58$\pm$0.07 & 4.11$\pm$1.93  & 46.16$\pm$6.88   \\
150   & 581012  & 7    & 48.76 & 71.13$\pm$0.06 & 8.42$\pm$2.06  & 51.13$\pm$4.08   \\
154   & 1000000 & 10   & 10.08 & 54.97$\pm$0.10 & 9.88$\pm$0.41  & 53.95$\pm$0.12   \\
156   & 1000000 & 5    & 30.01 & 45.51$\pm$0.25 & 25.14$\pm$2.56 & 42.81$\pm$0.16   \\
157   & 1000000 & 5    & 30.01 & 47.69$\pm$0.19 & 22.62$\pm$1.77 & 46.54$\pm$0.36   \\
158   & 1000000 & 5    & 30.01 & 47.84$\pm$0.16 & 22.23$\pm$2.11 & 46.98$\pm$0.07   \\
159   & 1000000 & 5    & 30.01 & 29.89$\pm$0.09 & 18.77$\pm$1.09 & 12.02$\pm$2.21   \\
160   & 1000000 & 5    & 30.01 & 30.52$\pm$0.06 & 16.17$\pm$1.79 & 13.51$\pm$1.70   \\
163   & 32      & 3    & 40.62 & 38.75$\pm$4.50 & 25.00$\pm$4.33 & 35.00$\pm$5.70   \\
171   & 339     & 21   & 24.78 & 23.82$\pm$2.13 & 3.53$\pm$1.11  & 7.21$\pm$2.14    \\
180   & 110393  & 7    & 46.82 & 64.34$\pm$0.11 & 10.34$\pm$1.86 & 40.44$\pm$3.94   \\
181   & 1484    & 10   & 31.20 & 47.05$\pm$1.22 & 6.31$\pm$1.17  & 21.07$\pm$2.09   \\
182   & 6430    & 6    & 23.81 & 82.20$\pm$0.36 & 18.99$\pm$1.40 & 76.87$\pm$0.88   \\
183   & 4177    & 28   & 16.50 & 19.38$\pm$0.50 & 2.79$\pm$0.73  & 6.60$\pm$1.40    \\
184   & 28056   & 18   & 16.23 & 23.94$\pm$0.32 & 5.71$\pm$0.63  & 6.95$\pm$0.83    \\
187   & 178     & 3    & 39.89 & 93.06$\pm$1.41 & 30.56$\pm$2.65 & 37.50$\pm$6.65   \\
188   & 736     & 5    & 29.08 & 49.59$\pm$1.06 & 18.85$\pm$1.46 & 21.96$\pm$3.81   \\
247   & 1000000 & 26   & 4.08  & 42.91$\pm$0.28 & 3.08$\pm$0.48  & 26.60$\pm$1.03   \\
248   & 1000000 & 7    & 32.36 & 54.00$\pm$0.06 & 8.53$\pm$3.26  & 51.59$\pm$0.33   \\
250   & 1000000 & 10   & 10.05 & 82.52$\pm$0.05 & 9.81$\pm$0.33  & 78.84$\pm$0.10   \\
252   & 1000000 & 10   & 10.04 & 96.43$\pm$0.01 & 9.63$\pm$0.59  & 95.62$\pm$0.04   \\
253   & 1000000 & 6    & 42.31 & 65.27$\pm$0.05 & 12.10$\pm$1.77 & 50.18$\pm$11.17  \\
254   & 1000000 & 10   & 10.03 & 74.63$\pm$0.04 & 11.11$\pm$1.23 & 72.07$\pm$0.09   \\
255   & 55296   & 3    & 42.62 & 52.14$\pm$0.17 & 32.04$\pm$2.04 & 41.94$\pm$1.93   \\
261   & 1000000 & 10   & 10.45 & 77.09$\pm$0.07 & 9.23$\pm$1.58  & 73.85$\pm$0.14   \\
263   & 1000000 & 6    & 30.46 & 97.50$\pm$0.01 & 17.02$\pm$2.75 & 96.54$\pm$0.01   \\
265   & 137781  & 7    & 35.40 & 53.82$\pm$0.14 & 15.45$\pm$2.60 & 38.90$\pm$2.25   \\
268   & 1000000 & 4    & 25.75 & 57.85$\pm$0.11 & 24.47$\pm$0.60 & 57.02$\pm$0.16   \\
271   & 1000000 & 3    & 33.84 & 88.70$\pm$0.04 & 33.14$\pm$0.07 & 87.99$\pm$0.12   \\
272   & 1000000 & 7    & 39.65 & 93.41$\pm$0.05 & 17.01$\pm$3.15 & 91.99$\pm$0.05   \\
285   & 194     & 8    & 30.93 & 31.00$\pm$2.36 & 12.25$\pm$2.25 & 19.75$\pm$2.68   \\
300   & 7797    & 26   & 3.85  & 80.51$\pm$0.46 & 3.44$\pm$0.27  & 30.18$\pm$2.08   \\
307   & 990     & 11   & 9.09  & 32.63$\pm$1.06 & 10.20$\pm$1.05 & 11.67$\pm$1.24   \\
313   & 531     & 48   & 10.36 & 8.52$\pm$1.08  & 1.20$\pm$0.38  & 2.31$\pm$0.63    \\
327   & 105     & 6    & 41.90 & 51.82$\pm$2.41 & 20.45$\pm$3.33 & 20.91$\pm$4.12   \\
328   & 105     & 6    & 41.90 & 50.45$\pm$3.77 & 13.64$\pm$3.55 & 25.45$\pm$3.61   \\
329   & 160     & 3    & 40.62 & 47.50$\pm$2.00 & 29.06$\pm$2.91 & 32.50$\pm$3.94   \\
338   & 155     & 4    & 31.61 & 41.88$\pm$2.03 & 20.62$\pm$2.30 & 21.25$\pm$2.44   \\
339   & 36      & 3    & 33.33 & 78.75$\pm$4.06 & 30.00$\pm$3.35 & 31.25$\pm$4.96   \\
340   & 52      & 3    & 44.23 & 49.17$\pm$5.33 & 31.67$\pm$2.86 & 35.00$\pm$5.26   \\
342   & 52      & 3    & 46.15 & 69.17$\pm$3.78 & 25.00$\pm$4.93 & 37.50$\pm$5.12   \\
372   & 10108   & 46   & 28.47 & 26.67$\pm$0.59 & 1.38$\pm$0.21  & 12.17$\pm$1.05   \\
375   & 9961    & 9    & 16.20 & 90.36$\pm$0.25 & 10.85$\pm$0.89 & 72.55$\pm$3.19   \\
377   & 600     & 6    & 16.67 & 64.58$\pm$2.41 & 12.42$\pm$1.62 & 21.75$\pm$3.10   \\
382   & 7019    & 8    & 27.61 & 36.33$\pm$0.69 & 10.42$\pm$0.72 & 20.98$\pm$1.52   \\
383   & 690     & 10   & 23.19 & 70.94$\pm$1.61 & 10.43$\pm$1.17 & 25.00$\pm$3.73   \\
385   & 927     & 7    & 37.97 & 90.16$\pm$1.16 & 15.48$\pm$2.25 & 27.31$\pm$4.33   \\
386   & 913     & 10   & 17.20 & 56.79$\pm$1.08 & 10.76$\pm$0.99 & 14.67$\pm$2.17   \\
387   & 414     & 9    & 31.88 & 70.48$\pm$2.09 & 7.62$\pm$1.57  & 19.52$\pm$2.89   \\
388   & 204     & 6    & 44.61 & 67.62$\pm$2.93 & 18.10$\pm$3.13 & 29.52$\pm$5.76   \\
389   & 2463    & 17   & 20.54 & 65.22$\pm$0.75 & 6.09$\pm$0.90  & 24.64$\pm$2.33   \\
390   & 9558    & 44   & 7.28  & 47.65$\pm$0.47 & 2.25$\pm$0.11  & 15.37$\pm$0.97   \\
391   & 1504    & 13   & 40.43 & 61.95$\pm$1.10 & 4.64$\pm$0.97  & 19.74$\pm$3.08   \\
392   & 1003    & 10   & 19.34 & 66.34$\pm$1.14 & 8.61$\pm$0.58  & 21.78$\pm$2.74   \\
393   & 3075    & 6    & 29.43 & 83.07$\pm$0.49 & 15.83$\pm$1.36 & 37.74$\pm$4.94   \\
394   & 918     & 10   & 16.23 & 61.74$\pm$1.13 & 9.89$\pm$0.89  & 14.78$\pm$2.23   \\
395   & 1657    & 25   & 22.39 & 47.74$\pm$0.83 & 2.35$\pm$0.31  & 13.40$\pm$2.00   \\
396   & 3204    & 6    & 29.43 & 81.59$\pm$0.49 & 15.67$\pm$1.49 & 35.42$\pm$4.69   \\
397   & 313     & 8    & 29.71 & 61.72$\pm$1.70 & 14.69$\pm$1.98 & 25.94$\pm$3.34   \\
398   & 1560    & 20   & 21.86 & 55.80$\pm$1.31 & 3.27$\pm$0.58  & 18.17$\pm$2.81   \\
399   & 11162   & 10   & 14.52 & 65.26$\pm$0.30 & 10.44$\pm$0.39 & 33.83$\pm$2.58   \\
400   & 878     & 10   & 27.68 & 79.43$\pm$1.16 & 7.78$\pm$1.03  & 23.41$\pm$4.39   \\
401   & 1050    & 10   & 15.71 & 59.90$\pm$1.03 & 9.05$\pm$0.86  & 15.52$\pm$2.27   \\
452   & 285     & 7    & 41.40 & 35.69$\pm$2.32 & 16.38$\pm$2.70 & 18.10$\pm$2.42   \\
457   & 27      & 4    & 44.44 & 51.67$\pm$6.86 & 30.00$\pm$7.42 & 23.33$\pm$5.82   \\
458   & 841     & 4    & 37.69 & 99.59$\pm$0.15 & 23.35$\pm$2.67 & 59.82$\pm$8.04   \\
460   & 379     & 4    & 37.20 & 50.53$\pm$1.64 & 21.71$\pm$1.74 & 30.13$\pm$2.37   \\
468   & 72      & 6    & 16.67 & 48.75$\pm$4.83 & 20.62$\pm$3.22 & 18.12$\pm$4.28   \\
469   & 797     & 6    & 19.45 & 19.31$\pm$0.91 & 16.75$\pm$1.01 & 16.88$\pm$1.07   \\
473   & 2796    & 6    & 24.32 & 63.55$\pm$1.21 & 15.73$\pm$1.30 & 24.27$\pm$3.77   \\
475   & 400     & 4    & 25.00 & 33.25$\pm$1.67 & 24.38$\pm$1.73 & 26.38$\pm$2.12   \\
554   & 70000   & 10   & 11.25 & 89.92$\pm$0.10 & 9.98$\pm$0.29  & 84.17$\pm$0.12   \\
679   & 1024    & 4    & 39.45 & 40.24$\pm$1.66 & 27.67$\pm$2.36 & 23.98$\pm$3.03   \\
685   & 130     & 5    & 20.00 & 33.85$\pm$3.00 & 20.38$\pm$2.68 & 20.38$\pm$3.14   \\
694   & 310     & 9    & 13.23 & 28.23$\pm$1.89 & 8.71$\pm$1.71  & 10.81$\pm$2.25   \\
952   & 214     & 6    & 35.51 & 50.91$\pm$2.66 & 16.14$\pm$3.36 & 20.91$\pm$3.19   \\
1041  & 3468    & 10   & 11.04 & 82.03$\pm$0.63 & 9.60$\pm$0.57  & 43.08$\pm$3.18   \\
1044  & 10936   & 3    & 38.97 & 50.41$\pm$0.37 & 32.98$\pm$1.29 & 43.37$\pm$1.46   \\
1079  & 95      & 5    & 28.42 & 46.50$\pm$5.49 & 17.50$\pm$2.90 & 16.50$\pm$3.62   \\
1080  & 113     & 5    & 45.13 & 32.50$\pm$2.88 & 15.00$\pm$3.36 & 17.50$\pm$2.42   \\
1081  & 89      & 4    & 48.31 & 38.33$\pm$3.89 & 23.33$\pm$2.93 & 23.89$\pm$4.03   \\
1083  & 214     & 7    & 32.24 & 28.41$\pm$3.20 & 14.32$\pm$2.52 & 19.32$\pm$3.08   \\
1088  & 383     & 9    & 40.47 & 47.56$\pm$3.16 & 10.00$\pm$1.88 & 12.56$\pm$2.23   \\
1102  & 96      & 9    & 47.92 & 60.50$\pm$3.35 & 17.50$\pm$4.29 & 20.00$\pm$4.06   \\
1106  & 190     & 14   & 15.79 & 22.63$\pm$2.91 & 7.89$\pm$1.21  & 6.05$\pm$1.36    \\
1109  & 96      & 11   & 23.96 & 43.50$\pm$4.08 & 11.00$\pm$1.99 & 19.00$\pm$4.06   \\
1115  & 151     & 3    & 34.44 & 47.19$\pm$2.71 & 28.12$\pm$2.14 & 29.69$\pm$1.82   \\
1177  & 1000000 & 22   & 24.07 & 49.84$\pm$0.16 & 6.52$\pm$2.14  & 41.14$\pm$0.99   \\
1183  & 1000000 & 6    & 23.84 & 80.99$\pm$0.13 & 23.14$\pm$3.79 & 76.32$\pm$0.21   \\
1185  & 1000000 & 3    & 40.11 & 92.69$\pm$0.04 & 33.58$\pm$2.51 & 92.60$\pm$0.02   \\
1186  & 1000000 & 5    & 28.98 & 46.57$\pm$0.17 & 17.83$\pm$2.19 & 46.59$\pm$0.12   \\
1209  & 1000000 & 11   & 9.14  & 30.53$\pm$0.22 & 9.14$\pm$0.55  & 27.91$\pm$1.06   \\
1214  & 1000000 & 9    & 16.08 & 77.46$\pm$0.09 & 8.53$\pm$0.90  & 74.37$\pm$0.35   \\
1233  & 945     & 7    & 14.81 & 35.58$\pm$1.48 & 14.42$\pm$0.65 & 16.11$\pm$1.30   \\
1378  & 1000000 & 26   & 4.08  & 34.15$\pm$0.36 & 3.98$\pm$0.44  & 20.64$\pm$4.44   \\
1379  & 1000000 & 26   & 4.08  & 20.26$\pm$0.57 & 3.64$\pm$0.19  & 5.09$\pm$0.38    \\
1380  & 1000000 & 26   & 4.08  & 13.97$\pm$0.46 & 3.79$\pm$0.16  & 4.01$\pm$1.18    \\
1381  & 1000000 & 26   & 4.08  & 39.52$\pm$0.34 & 3.85$\pm$0.40  & 26.77$\pm$6.91   \\
1382  & 1000000 & 26   & 4.08  & 34.60$\pm$0.34 & 4.02$\pm$0.43  & 26.70$\pm$2.34   \\
1383  & 1000000 & 26   & 4.08  & 29.49$\pm$0.20 & 4.08$\pm$0.28  & 16.45$\pm$5.71   \\
1384  & 1000000 & 26   & 4.08  & 41.35$\pm$0.20 & 3.86$\pm$0.14  & 22.29$\pm$6.30   \\
1385  & 1000000 & 26   & 4.08  & 37.62$\pm$0.14 & 3.92$\pm$0.50  & 21.30$\pm$5.19   \\
1386  & 1000000 & 26   & 4.08  & 33.98$\pm$0.28 & 4.43$\pm$0.22  & 30.04$\pm$0.80   \\
1387  & 1000000 & 24   & 24.14 & 74.78$\pm$0.05 & 2.85$\pm$1.01  & 66.79$\pm$0.07   \\
1388  & 1000000 & 24   & 24.14 & 55.58$\pm$0.05 & 4.55$\pm$0.95  & 49.54$\pm$0.88   \\
1389  & 1000000 & 24   & 24.14 & 44.84$\pm$0.11 & 3.43$\pm$0.50  & 34.55$\pm$1.30   \\
1390  & 1000000 & 24   & 24.14 & 77.74$\pm$0.11 & 1.94$\pm$0.42  & 71.71$\pm$0.14   \\
1391  & 1000000 & 24   & 24.14 & 74.80$\pm$0.05 & 2.80$\pm$0.79  & 67.01$\pm$0.31   \\
1392  & 1000000 & 24   & 24.14 & 69.23$\pm$0.07 & 6.82$\pm$2.16  & 62.90$\pm$0.54   \\
1393  & 1000000 & 7    & 32.36 & 46.65$\pm$0.12 & 20.86$\pm$4.63 & 36.27$\pm$1.42   \\
1394  & 1000000 & 7    & 32.36 & 35.02$\pm$0.05 & 15.67$\pm$3.38 & 22.66$\pm$2.35   \\
1395  & 1000000 & 7    & 32.36 & 33.36$\pm$0.07 & 7.12$\pm$2.60  & 4.77$\pm$2.63    \\
1396  & 1000000 & 7    & 32.36 & 55.51$\pm$0.20 & 17.11$\pm$2.66 & 53.13$\pm$0.18   \\
1397  & 1000000 & 7    & 32.36 & 46.56$\pm$0.13 & 14.76$\pm$4.15 & 37.56$\pm$1.30   \\
1398  & 1000000 & 7    & 32.36 & 37.47$\pm$0.10 & 8.33$\pm$2.70  & 21.45$\pm$1.93   \\
1399  & 1000000 & 7    & 32.36 & 53.28$\pm$0.09 & 12.66$\pm$2.06 & 49.73$\pm$0.46   \\
1400  & 549796  & 7    & 32.43 & 53.06$\pm$0.05 & 20.55$\pm$2.25 & 46.79$\pm$0.41   \\
1401  & 1000000 & 7    & 32.36 & 46.72$\pm$0.11 & 11.69$\pm$2.61 & 37.31$\pm$0.35   \\
1413  & 150     & 3    & 33.33 & 88.67$\pm$2.63 & 31.00$\pm$2.22 & 35.33$\pm$7.57   \\
1457  & 1500    & 50   & 2.00  & 11.60$\pm$0.64 & 1.90$\pm$0.32  & 2.30$\pm$0.31    \\
1459  & 10218   & 10   & 13.86 & 32.69$\pm$0.39 & 8.52$\pm$0.57  & 9.77$\pm$1.22    \\
1465  & 106     & 6    & 20.75 & 24.55$\pm$3.02 & 13.64$\pm$1.76 & 14.55$\pm$2.60   \\
1466  & 2126    & 10   & 27.23 & 99.77$\pm$0.08 & 7.89$\pm$1.31  & 62.93$\pm$5.20   \\
1468  & 1080    & 9    & 11.11 & 87.41$\pm$0.84 & 10.32$\pm$0.85 & 41.94$\pm$3.52   \\
1472  & 768     & 37   & 9.64  & 13.90$\pm$1.14 & 3.83$\pm$0.73  & 2.73$\pm$0.78    \\
1475  & 6118    & 6    & 41.75 & 46.27$\pm$0.45 & 17.03$\pm$1.63 & 24.70$\pm$2.06   \\
1476  & 13910   & 6    & 21.63 & 94.80$\pm$0.25 & 17.51$\pm$1.00 & 87.76$\pm$1.11   \\
1477  & 13910   & 6    & 21.63 & 96.46$\pm$0.18 & 15.03$\pm$1.34 & 84.52$\pm$1.91   \\
1478  & 10299   & 6    & 18.88 & 95.02$\pm$0.29 & 17.29$\pm$1.27 & 91.06$\pm$2.23   \\
1481  & 28056   & 18   & 16.23 & 23.34$\pm$0.31 & 5.33$\pm$0.61  & 7.39$\pm$1.05    \\
1482  & 340     & 30   & 4.71  & 10.59$\pm$1.05 & 3.24$\pm$0.65  & 3.24$\pm$0.62    \\
1483  & 164860  & 11   & 33.05 & 38.96$\pm$0.18 & 7.54$\pm$1.30  & 17.07$\pm$1.56   \\
1491  & 1600    & 100  & 1.00  & 6.81$\pm$0.50  & 0.88$\pm$0.14  & 1.66$\pm$0.24    \\
1492  & 1600    & 100  & 1.00  & 2.47$\pm$0.30  & 0.75$\pm$0.18  & 1.06$\pm$0.24    \\
1493  & 1599    & 100  & 1.00  & 6.31$\pm$0.36  & 0.88$\pm$0.17  & 1.28$\pm$0.24    \\
1497  & 5456    & 4    & 40.41 & 66.40$\pm$0.52 & 24.53$\pm$3.06 & 44.01$\pm$4.83   \\
1499  & 210     & 3    & 33.33 & 70.95$\pm$3.87 & 32.62$\pm$1.66 & 25.95$\pm$4.74   \\
1500  & 210     & 3    & 33.33 & 75.71$\pm$2.84 & 32.86$\pm$1.90 & 27.86$\pm$4.63   \\
1501  & 1593    & 10   & 10.17 & 80.28$\pm$0.62 & 9.12$\pm$0.63  & 30.31$\pm$3.21   \\
1503  & 263256  & 10   & 10.06 & 10.13$\pm$0.03 & 10.03$\pm$0.05 & 10.03$\pm$0.05   \\
1508  & 403     & 5    & 32.01 & 67.68$\pm$1.44 & 19.39$\pm$2.34 & 24.63$\pm$3.92   \\
1509  & 149332  & 22   & 14.73 & 23.68$\pm$0.41 & 5.07$\pm$0.67  & 7.76$\pm$0.43    \\
1512  & 200     & 5    & 28.00 & 25.00$\pm$1.66 & 20.75$\pm$2.07 & 21.00$\pm$1.82   \\
1513  & 123     & 5    & 39.02 & 31.54$\pm$3.12 & 19.23$\pm$3.24 & 26.92$\pm$3.46   \\
1514  & 360     & 10   & 10.00 & 65.56$\pm$2.37 & 8.33$\pm$1.08  & 24.72$\pm$3.01   \\
1515  & 571     & 20   & 10.51 & 18.28$\pm$1.29 & 6.29$\pm$0.73  & 7.07$\pm$1.21    \\
1516  & 88      & 4    & 38.64 & 58.89$\pm$4.24 & 21.11$\pm$3.13 & 16.11$\pm$3.47   \\
1517  & 47      & 5    & 42.55 & 31.00$\pm$4.12 & 17.00$\pm$3.54 & 18.00$\pm$5.08   \\
1518  & 47      & 4    & 42.55 & 39.00$\pm$4.79 & 13.00$\pm$2.56 & 21.00$\pm$4.58   \\
1520  & 164     & 5    & 28.66 & 42.94$\pm$2.28 & 19.41$\pm$2.73 & 22.35$\pm$4.27   \\
1523  & 215     & 3    & 46.05 & 82.95$\pm$1.84 & 32.95$\pm$2.18 & 37.50$\pm$7.14   \\
1525  & 5456    & 4    & 40.41 & 74.57$\pm$0.56 & 23.72$\pm$3.21 & 39.73$\pm$4.93   \\
1548  & 2500    & 3    & 46.92 & 47.08$\pm$0.86 & 35.46$\pm$3.66 & 36.28$\pm$1.41   \\
1549  & 742     & 8    & 22.10 & 18.20$\pm$0.99 & 11.87$\pm$0.98 & 12.80$\pm$1.31   \\
1551  & 400     & 8    & 27.75 & 16.50$\pm$1.64 & 14.12$\pm$1.72 & 12.50$\pm$1.26   \\
1552  & 1100    & 5    & 27.73 & 29.23$\pm$0.89 & 19.86$\pm$1.63 & 21.86$\pm$1.69   \\
1553  & 700     & 3    & 35.00 & 38.71$\pm$1.18 & 30.43$\pm$1.06 & 34.93$\pm$1.16   \\
1554  & 500     & 5    & 38.40 & 32.20$\pm$1.31 & 20.40$\pm$2.49 & 21.00$\pm$2.17   \\
1555  & 976     & 8    & 23.57 & 15.26$\pm$1.09 & 11.43$\pm$0.79 & 12.60$\pm$0.83   \\
1568  & 12958   & 4    & 33.34 & 89.77$\pm$0.25 & 23.83$\pm$2.77 & 83.15$\pm$1.89   \\
1596  & 580943  & 7    & 48.76 & 69.45$\pm$0.07 & 18.54$\pm$3.50 & 36.97$\pm$2.49   \\
4552  & 5665    & 102  & 8.88  & 51.48$\pm$1.22 & 0.39$\pm$0.12  & 15.04$\pm$1.60   \\
40927 & 60000   & 10   & 10.00 & 30.01$\pm$0.69 & 9.98$\pm$0.17  & 13.24$\pm$1.30   \\
40966 & 1080    & 8    & 13.89 & 47.27$\pm$2.21 & 13.52$\pm$1.14 & 22.18$\pm$2.34   \\
40971 & 1000    & 30   & 8.00  & 10.40$\pm$0.71 & 4.45$\pm$0.63  & 4.75$\pm$0.83    \\
40979 & 2000    & 10   & 10.00 & 92.40$\pm$0.46 & 10.32$\pm$0.97 & 38.15$\pm$3.68   \\
40982 & 1941    & 7    & 34.67 & 64.08$\pm$0.61 & 13.31$\pm$1.97 & 31.62$\pm$3.78   \\
40984 & 2310    & 7    & 14.29 & 79.94$\pm$0.81 & 16.15$\pm$1.25 & 50.71$\pm$4.97   \\
40985 & 45781   & 20   & 6.35  & 6.28$\pm$0.13  & 5.05$\pm$0.22  & 5.33$\pm$0.33    \\
40996 & 70000   & 10   & 10.00 & 83.51$\pm$0.14 & 9.61$\pm$0.45  & 80.84$\pm$0.16   \\
41002 & 5880    & 3    & 44.47 & 97.57$\pm$0.16 & 32.84$\pm$1.91 & 83.48$\pm$5.57   \\
41003 & 5880    & 3    & 49.61 & 87.84$\pm$0.40 & 34.97$\pm$2.99 & 65.33$\pm$6.96   \\
41004 & 4704    & 3    & 44.20 & 97.44$\pm$0.18 & 31.80$\pm$1.46 & 92.30$\pm$3.63   \\
41039 & 131600  & 47   & 2.13  & 51.52$\pm$0.18 & 2.25$\pm$0.08  & 16.85$\pm$0.86   \\
41081 & 44557   & 10   & 18.82 & 15.47$\pm$0.80 & 8.84$\pm$0.44  & 10.35$\pm$0.38   \\
41082 & 9298    & 10   & 16.7  & 91.12$\pm$0.20 & 11.21$\pm$0.82 & 77.66$\pm$2.29   \\
41083 & 400     & 40   & 2.50  & 4.88$\pm$1.33  & 2.88$\pm$0.69  & 2.00$\pm$0.45    \\
41084 & 575     & 20   & 8.35  & 29.91$\pm$1.41 & 4.14$\pm$0.59  & 6.29$\pm$0.59    \\
41163 & 10000   & 5    & 20.49 & 87.03$\pm$0.45 & 19.44$\pm$0.55 & 66.98$\pm$4.69   \\
41164 & 8237    & 7    & 23.39 & 54.04$\pm$0.37 & 14.51$\pm$0.52 & 43.50$\pm$2.91   \\
41165 & 10000   & 10   & 10.43 & 27.56$\pm$0.45 & 10.08$\pm$0.32 & 10.49$\pm$0.42   \\
41166 & 58310   & 10   & 21.96 & 53.11$\pm$0.44 & 11.71$\pm$1.12 & 45.36$\pm$1.32   \\
41167 & 416188  & 355  & 0.59  & 28.28$\pm$0.26 & 0.27$\pm$0.02  & 2.06$\pm$0.17    \\
41168 & 83733   & 4    & 46.01 & 62.84$\pm$0.15 & 24.93$\pm$2.63 & 48.94$\pm$3.76   \\
41169 & 65196   & 100  & 6.14  & 19.49$\pm$0.24 & 0.65$\pm$0.05  & 7.34$\pm$0.61    \\
41511 & 150     & 3    & 33.33 & 84.67$\pm$3.09 & 34.33$\pm$2.59 & 52.00$\pm$7.14   \\
41568 & 150     & 3    & 33.33 & 85.00$\pm$2.98 & 32.33$\pm$1.90 & 44.67$\pm$7.59   \\
41583 & 150     & 3    & 33.33 & 82.67$\pm$3.00 & 36.67$\pm$1.91 & 39.67$\pm$8.47   \\
41919 & 527     & 4    & 39.47 & 50.57$\pm$1.34 & 23.40$\pm$1.99 & 31.51$\pm$4.68   \\
41939 & 218     & 5    & 38.53 & 35.23$\pm$2.31 & 23.64$\pm$3.21 & 23.64$\pm$2.22   \\
41950 & 150     & 3    & 33.33 & 80.00$\pm$2.67 & 32.33$\pm$1.84 & 25.67$\pm$6.62   \\
41960 & 523590  & 144  & 25.08 & 38.46$\pm$0.79 & 0.56$\pm$0.14  & 15.89$\pm$2.20   \\
41972 & 9144    & 8    & 44.29 & 63.16$\pm$1.64 & 12.44$\pm$3.11 & 26.07$\pm$4.36   \\
41981 & 296     & 14   & 30.74 & 26.33$\pm$1.65 & 4.83$\pm$1.09  & 12.00$\pm$1.77   \\
41982 & 70000   & 10   & 10.00 & 75.86$\pm$0.12 & 10.11$\pm$0.25 & 70.14$\pm$0.14   \\
41986 & 51839   & 43   & 5.79  & 87.55$\pm$0.17 & 3.23$\pm$0.29  & 50.86$\pm$1.25   \\
41988 & 51839   & 43   & 5.79  & 89.76$\pm$0.14 & 2.97$\pm$0.31  & 56.77$\pm$1.73   \\
41989 & 51839   & 43   & 5.79  & 88.43$\pm$0.17 & 1.61$\pm$0.16  & 45.18$\pm$1.15   \\
41990 & 51839   & 43   & 5.79  & 15.89$\pm$0.13 & 1.54$\pm$0.16  & 8.34$\pm$0.42    \\
41991 & 270912  & 49   & 2.58  & 53.96$\pm$0.13 & 2.17$\pm$0.08  & 26.62$\pm$0.65   \\
41997 & 150     & 3    & 33.33 & 89.00$\pm$2.07 & 30.00$\pm$1.73 & 42.33$\pm$8.44   \\
42003 & 150     & 3    & 33.33 & 82.67$\pm$3.00 & 34.67$\pm$2.19 & 16.67$\pm$5.30   \\
42011 & 150     & 3    & 33.33 & 81.67$\pm$2.90 & 33.67$\pm$3.18 & 27.33$\pm$7.87   \\
42016 & 150     & 3    & 33.33 & 83.67$\pm$2.96 & 32.00$\pm$2.39 & 32.67$\pm$7.66   \\
42021 & 150     & 3    & 33.33 & 82.00$\pm$3.17 & 36.00$\pm$3.14 & 37.00$\pm$6.77   \\
42026 & 150     & 3    & 33.33 & 83.00$\pm$3.28 & 31.33$\pm$2.27 & 38.33$\pm$7.01   \\
42031 & 150     & 3    & 33.33 & 84.67$\pm$2.83 & 33.67$\pm$2.69 & 27.33$\pm$6.53   \\
42036 & 150     & 3    & 33.33 & 85.67$\pm$2.68 & 37.00$\pm$2.23 & 44.67$\pm$7.57   \\
42041 & 150     & 3    & 33.33 & 83.33$\pm$2.81 & 29.33$\pm$2.08 & 41.33$\pm$8.06   \\
42046 & 150     & 3    & 33.33 & 86.67$\pm$2.58 & 32.67$\pm$3.09 & 30.33$\pm$6.90   \\
42051 & 150     & 3    & 33.33 & 82.67$\pm$3.03 & 30.00$\pm$2.33 & 29.33$\pm$7.19   \\
42056 & 150     & 3    & 33.33 & 80.33$\pm$3.35 & 33.00$\pm$2.52 & 28.67$\pm$7.29   \\
42066 & 150     & 3    & 33.33 & 82.00$\pm$3.50 & 31.67$\pm$2.20 & 40.00$\pm$5.85   \\
42071 & 150     & 3    & 33.33 & 83.33$\pm$2.56 & 29.67$\pm$2.60 & 43.67$\pm$6.95   \\
42098 & 150     & 3    & 33.33 & 84.00$\pm$2.92 & 29.67$\pm$2.33 & 37.00$\pm$7.31   \\
42140 & 9927    & 10   & 19.10 & 13.87$\pm$0.77 & 9.21$\pm$0.80  & 10.69$\pm$0.54   \\
42141 & 49644   & 10   & 19.10 & 15.84$\pm$0.86 & 11.59$\pm$0.89 & 11.03$\pm$0.54   \\
42186 & 150     & 3    & 33.33 & 82.33$\pm$2.92 & 32.67$\pm$2.26 & 28.00$\pm$7.27   \\
42261 & 150     & 3    & 33.33 & 81.67$\pm$3.30 & 33.33$\pm$2.94 & 24.00$\pm$5.88   \\
42345 & 70340   & 3    & 48.88 & 61.16$\pm$0.14 & 22.11$\pm$4.71 & 45.74$\pm$1.31   \\
42396 & 108000  & 1000 & 0.10  & 4.78$\pm$0.09  & 0.11$\pm$0.01  & 0.42$\pm$0.03    \\
42468 & 830000  & 5    & 20.22 & 68.18$\pm$0.14 & 21.15$\pm$1.19 & 60.11$\pm$0.39   \\
42532 & 2778    & 10   & 27.29 & 45.90$\pm$0.91 & 9.95$\pm$0.78  & 21.89$\pm$2.99   \\
42544 & 265     & 8    & 17.74 & 42.04$\pm$1.68 & 14.81$\pm$1.76 & 15.00$\pm$2.31   \\
42585 & 344     & 3    & 44.19 & 75.14$\pm$1.62 & 39.43$\pm$2.45 & 26.57$\pm$6.01   \\
42700 & 150     & 3    & 33.33 & 87.33$\pm$2.70 & 33.67$\pm$1.53 & 22.00$\pm$6.15   \\
42718 & 1000000 & 4    & 25.75 & 57.78$\pm$0.10 & 24.35$\pm$1.67 & 57.13$\pm$0.07   \\
42793 & 75      & 4    & 40.00 & 71.25$\pm$3.07 & 23.12$\pm$3.22 & 34.38$\pm$5.44   \\
43859 & 150     & 3    & 33.33 & 84.33$\pm$2.99 & 30.67$\pm$2.47 & 32.33$\pm$7.25   \\
43875 & 150     & 3    & 33.33 & 85.33$\pm$2.34 & 32.00$\pm$2.34 & 42.33$\pm$7.34  
\end{longtable}
} 

\subsection{Ablation Analysis on OpenML Datasets}

This section provides an ablation study based on the results of OpenML datasets in Table~\ref{tab:openml_action_exclusive} and Figure~\ref{fig:exp_openml_balanced}.
Both Table~\ref{tab:openml_action_exclusive} and Figure~\ref{fig:exp_openml_balanced} compares the performance of \newIGL, CB, and \oldIGL on all OpenML datasets with balanced action distributions, and that with large sample size (sample size $\geq$ MNIST). We can observe that \newIGL almost always outperforms \oldIGL. 
We now demonstrate that the performance of \newIGL is affected by: 1) the label noise; 2) the size of datasets (expected from theory), by the following observation.

At first, it is easy to notice that the performance of \newIGL is significantly benefited from the sample size, by comparing with the performance of \newIGL between all datasets and datasets with a large sample size (also suggested in Table~\ref{tab:openml_action_exclusive}), which follows the prediction of Theorem~\ref{thm:main_thm}. 
On the other hand, we study how the label noise affects the performance of \newIGL. To ablate the effect from the sample size, we only study the results in the dataset with large sample size. In this case, we consider the performance of CB to indicate the label noise, i.e., datasets with high CB performance means small label noise, and vice versa.
We can observe from Figure~\ref{fig:openml_large_N} that the performance gap between CB and \newIGL increases considerably with the label noise increasing (i.e., CB performance decreasing). This suggests that \newIGL is more sensitive to the label noise compared with CB. It is unclear if this is an information-theoretical difficulty due to the setting without explicit reward, and we leave further investigation along this line as future work.

\begin{figure*}
\begin{subfigure}{0.5\textwidth}
\centering
\includegraphics[width=.99\textwidth]{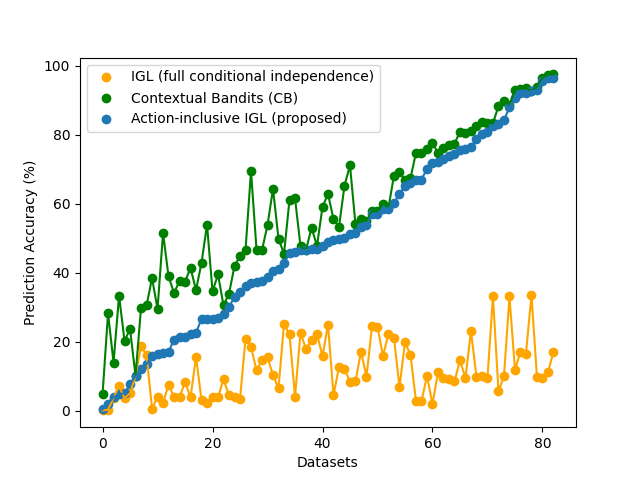}
\caption{Datasets with $K\geq3$, $N\geq70000$}
\label{fig:openml_large_N}
\end{subfigure}
~
\begin{subfigure}{0.5\textwidth}
\centering
\includegraphics[width=.99\textwidth]{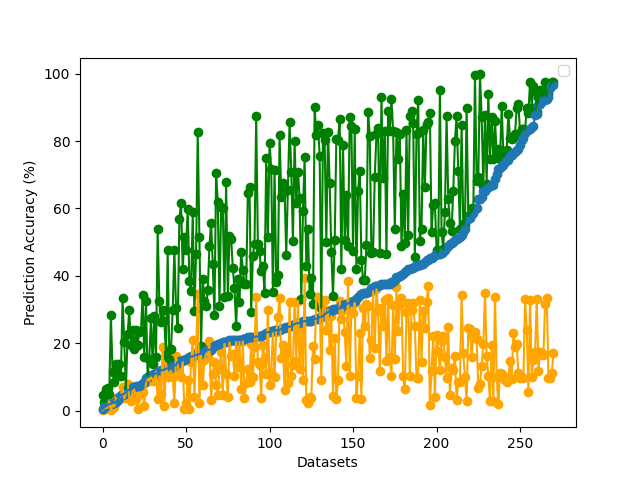}
\caption{Datasets with $K\geq3$, no constraints on $N$}
\label{fig:openml_all}
\end{subfigure}
\caption{Average performance on datasets with balanced action distributions from the OpenML benchmarking suite. For legibility of the figure, we do not include standard errors here. Standard errors for individual datasets are reported in Table~\ref{tab:all_openml}. $K$ is the size of the action set and $N$ is the sample size.}
\label{fig:exp_openml_balanced}
\end{figure*}

\subsection{Analyzing OpenML Dataset Meta-Properties with respect to AI-IGL's Performance}
\label{sec:feat_analysis}

In order to better understand which features of a dataset make them amenable to high IGL performance, we systematically analyzed the features of OpenML datasets. Specifically, we measured a total of 15
features/meta-properties for all OpenML datasets and compared the accuracy of AI-IGL's success relative to CB. We first explain the first 12, and explanation for the additional 3 measures follow.

We collected a total of 49 features following the findings of~\citet{lorena2019complex,torra2008modeling,reif2014automatic,abdelmessih2010landmarking}. Based on preliminary analysis. We narrowed these features down to the following 12 features for further investigation of AI-IGL on OpenML datasets, because they had the least correlations with one another:
\begin{itemize}
    \item Accuracy of the 1-nearest neighbor classifier on the dataset (1nn\_accuracy)~\citep{reif2014automatic,abdelmessih2010landmarking}
    \item Best single decision node accuracy created using the feature attribute with the highest information gain (best\_node\_accuracy)~\citep{reif2014automatic,abdelmessih2010landmarking}
    \item The number of 0/1 features when the dataset is one-hot encoded (feature\_onehot\_count)~\citep{lorena2019complex}
    \item Ratio of sample dimensionality by sample count (instance\_per\_feature)~\citep{lorena2019complex} 
    \item The ratio of the class distribution entropy and the maximum entropy for the uniform distribution over classes (class\_entropy\_N)
    \item Maximum Fisher discriminant ratio %
    (max\_fisher\_discrim)~\citep{lorena2019complex}
    \item The percentage of values in the feature matrix that are non-zero (max\_single\_feature\_eff)~\citep{lorena2019complex}
    \item Mutual information mean (mutual\_XY\_info\_mean)~\citep{torra2008modeling,reif2014automatic}
    \item Naive Bayes accuracy (naive\_bayes\_accuracy)~\citep{reif2014automatic,abdelmessih2010landmarking}
    \item Noise to signal ratio (noise\_signal\_ratio)~\citep{torra2008modeling}
    \item %
    Number of principal components needed to represent 95\% of data variability (pca\_dims\_95)~\citep{lorena2019complex}
    \item Percent of data variance explained by the top principal component (pca\_top\_1\_percent)~\citep{lorena2019complex}
\end{itemize}

In addition to these 12 features, we added three additional features as follows. Compared to the typical CB guarantee, \newIGL needs one more $K$ factor in its theoretical guarantees (Theorem~\ref{thm:main_thm}). Since $K$ factors can be improved under some specific choice of function class (see discussion in Section~\ref{sec:algo_theory}), four additional features of the dataset were used to predict the relative performance of \newIGL:
\begin{itemize}
    \item $N$ (n)
    \item $N/\sqrt{K}$ (n\_by\_sqrt\_k)
    \item $N/K$ (n\_by\_k)
\end{itemize}

We used a binary random forest classifier to predict the success of AI-IGL's performance relative to CB. If the relative performance is $\geq0.7$, we label it as a success. In Table~\ref{tab:f1}, we report the F1 scores for the success and failure classes defined in this way for a binary random forest classifier with 100 trees (evaluated using 10-fold cross validation on all 271 OpenML datasets). Based on further analysis of the importance weight for each feature computed using the information gain metric (Figure~\ref{fig:importance_a}), we created smaller subsets of the datasets to further analyze datasets representative of realistic interaction datasets with small sample sizes. We also report the average F1 scores for random forest classifiers trained and evaluated in a similar manner in rows 2,3 in Table~\ref{tab:f1}.%

\begin{table*}[!htb]
    \centering
    \begin{tabular}{|c|c|c|c|}
    \hline
         Dataset properties & AI-IGL accuracy$<70\%$    & \begin{tabular}[x]{@{}c@{}}AI-IGL accuracy$\geq70\%$\\(Success) \end{tabular}  & \begin{tabular}[x]{@{}c@{}}Average F1 score\\(both classes)\end{tabular} \\
        \hline
        $K\geq3$ & 0.91 & 0.83 & 0.87\\
        \hline
        $N/K\leq1000$ & 0.95 & 0.46 & 0.71\\
        \hline
        $N/K\leq200$ & 0.97 & 0.56 & 0.77\\
        \hline
    \end{tabular}
    \caption{F1 scores of binary random forest classifiers predicting the success of AI-IGL relative to CB on different datasets using curated features/meta-properties.
    }
    \label{tab:f1}
\end{table*}

We also computed the feature importance for each of the random forest classifiers to determine which feature is most predictive of AI-IGL's success relative to CB. Feature importance was computed using the information gain metric. We used python's scikit-learn library to implement the random forest classifiers as well as feature importance values. 

We present feature importance plots for three different subsets of datasets in Figure~\ref{fig:importance}. We find $N/K$ to be the most predictive feature of AI-IGL's relative performance (Figure~\ref{fig:importance_a}).  It can alone predict its performance with an average F1 score of 0.79 under the same experimental setup. However, for datasets with a small value of $N/K$ ($\leq1000$, $\leq200$), there is high variability in relative performance. Using such a subset of datasets, we find maximum Fisher discriminant~\citep{lorena2019complex} (a measure of classification complexity that quantifies the informativeness of a given sample) to be the most predictive of relative performance (Figure~\ref{fig:importance_b}, Figure~\ref{fig:importance_c}). 

This finding identifies clear measures of small and large datasets that can predict whether AI-IGL can match CB performance for any given dataset. It can also help researchers improve the design of novel applications of IGL, e.g., in HCI and BCI, by ensuring the resulting dataset’s features are amenable to high performance.

\begin{figure*}[t!]
\begin{subfigure}{0.45\textwidth}
\centering
\includegraphics[width=0.97\textwidth]{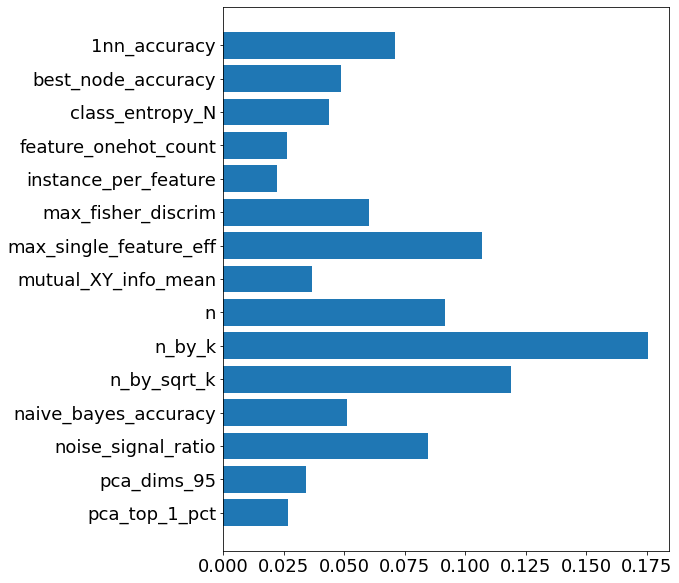}
\caption{Feature importance for 271 datasets with $K\geq3$)}
\label{fig:importance_a}
\end{subfigure}
~
\begin{subfigure}{0.45\textwidth}
\centering
\includegraphics[width=.97\textwidth]{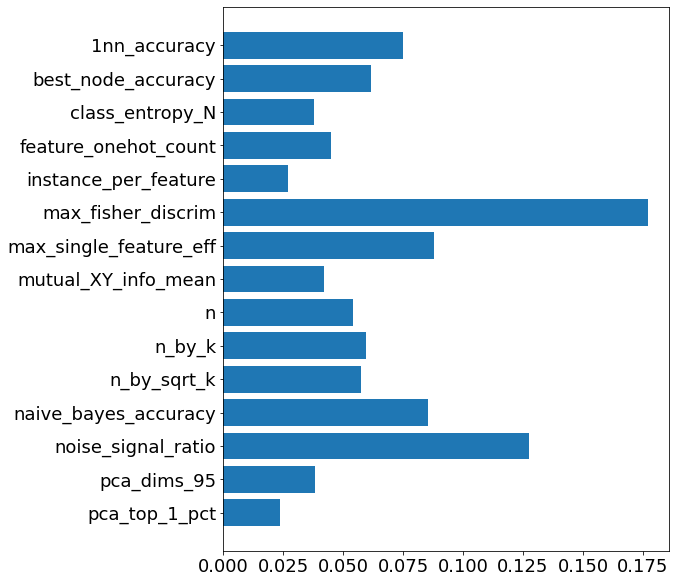}
\caption{Feature importance for 155 datasets with $N/K\leq1000$}
\label{fig:importance_b}
\end{subfigure}
~
\begin{subfigure}{0.45\textwidth}
\centering
\includegraphics[width=.97\textwidth]{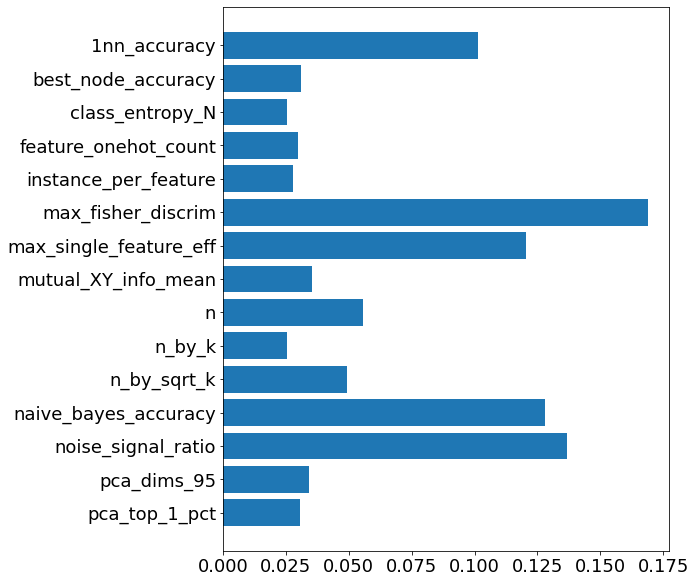}
\caption{Feature importance for 128 datasets with $N/K\leq200$}
\label{fig:importance_c}
\end{subfigure}
\caption{Feature importance over dataset meta-features used to classify the success of AI-IGL. We find that for across all datasets $N/K$ is the most informative feature (Figure~\ref{fig:importance_a}), whereas maximum Fisher discriminant ratio is the most important feature over datasets with smaller values of $N/K$ (Figure~\ref{fig:importance_b} and~\ref{fig:importance_c}).}
\label{fig:importance}
\end{figure*}

For the features considered relatively more important than others as shown in Figure~\ref{fig:importance_a}, \ref{fig:importance_b} and~\ref{fig:importance_c}), we also present how the relative performance of AI-IGL varies for different values of such features. Visualizations for different features across all 271 datasets are shown in Figure~\ref{fig:features_full}, across the subset of 155 datasets with $N/K\leq 1000$ are shown in Figure~\ref{fig:features_1000}, and across the subset of 128 datasets with $N/K\leq 200$ are shown in Figure~\ref{fig:features_200}.

\begin{figure*}[t!]
\begin{subfigure}{0.49\textwidth}
\centering
\includegraphics[width=0.95\textwidth]{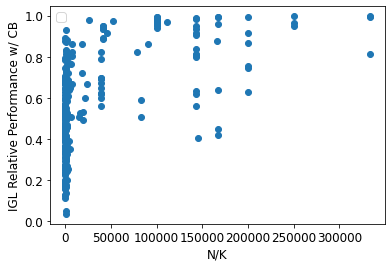}
\caption{}
\end{subfigure}
~
\begin{subfigure}{0.49\textwidth}
\centering
\includegraphics[width=.95\textwidth]{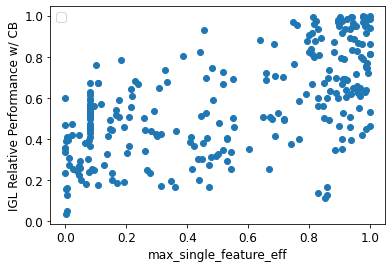}
\caption{}
\end{subfigure}
\\
\begin{subfigure}{0.49\textwidth}
\centering
\includegraphics[width=.95\textwidth]{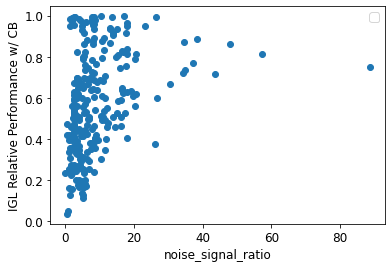}
\caption{}
\end{subfigure}
~
\begin{subfigure}{0.49\textwidth}
\centering
\includegraphics[width=.95\textwidth]{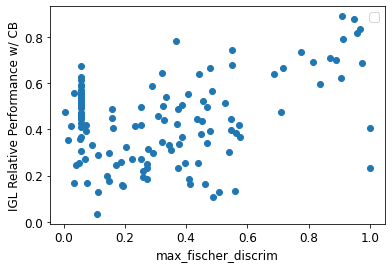}
\caption{}
\end{subfigure}
\\
\begin{subfigure}{0.49\textwidth}
\centering
\includegraphics[width=.95\textwidth]{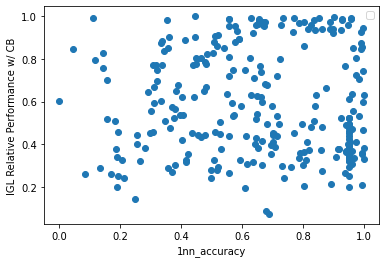}
\caption{}
\end{subfigure}
~
\begin{subfigure}{0.49\textwidth}
\centering
\includegraphics[width=.95\textwidth]{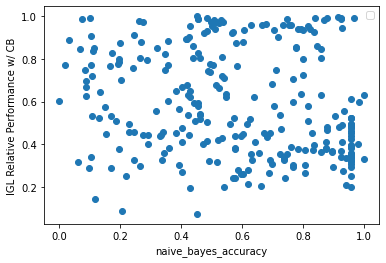}
\caption{}
\end{subfigure}
\\
\begin{subfigure}{0.49\textwidth}
\centering
\includegraphics[width=.95\textwidth]{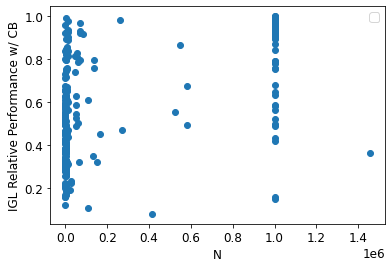}
\caption{}
\end{subfigure}
~
\begin{subfigure}{0.49\textwidth}
\centering
\includegraphics[width=.95\textwidth]{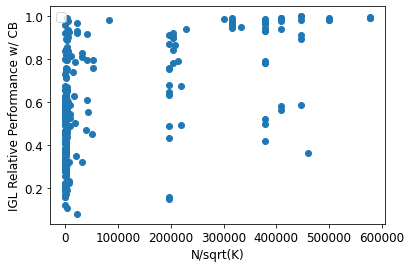}
\caption{}
\end{subfigure}
\caption{AI-IGL's relative performance w.r.t.~CB versus different feature values for all 271 datasets with $K\geq3$. $N/K$ is the most important feature for predicting AI-IGL's relative performance as shown in Figure~\ref{fig:importance_a}. Here we observe that the variability in AI-IGL's relative performance decreases as $N/K$ increases. Datasets with a smaller values of $N/K$ are analyzed separately (discussed in Figure~\ref{fig:features_1000} and Figure~\ref{fig:features_200}). }
\label{fig:features_full}
\end{figure*}

\begin{figure*}[t!]
\begin{subfigure}{0.49\textwidth}
\centering
\includegraphics[width=0.95\textwidth]{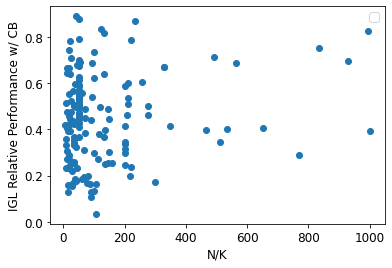}
\caption{}
\end{subfigure}
~
\begin{subfigure}{0.49\textwidth}
\centering
\includegraphics[width=.95\textwidth]{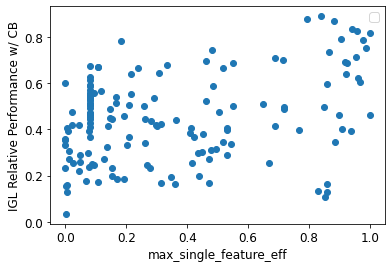}
\caption{}
\end{subfigure}
\\
\begin{subfigure}{0.49\textwidth}
\centering
\includegraphics[width=.95\textwidth]{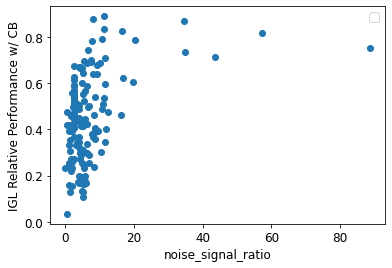}
\caption{}
\end{subfigure}
~
\begin{subfigure}{0.49\textwidth}
\centering
\includegraphics[width=.95\textwidth]{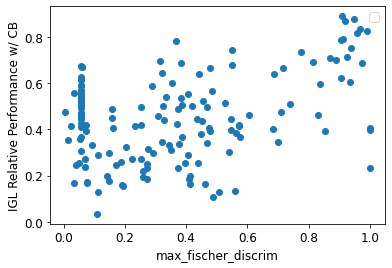}
\caption{}
\label{fig:features_1000_d}
\end{subfigure}
\\
\begin{subfigure}{0.49\textwidth}
\centering
\includegraphics[width=.95\textwidth]{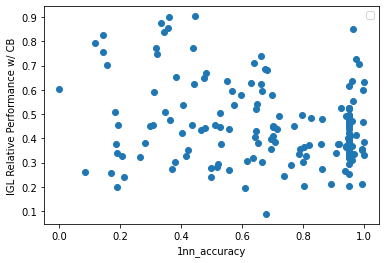}
\caption{}
\end{subfigure}
~
\begin{subfigure}{0.49\textwidth}
\centering
\includegraphics[width=.95\textwidth]{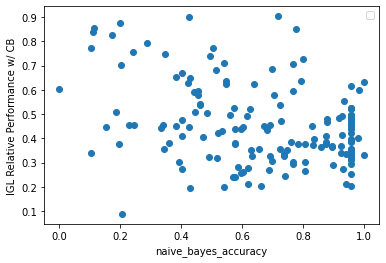}
\caption{}
\end{subfigure}
\\
\begin{subfigure}{0.49\textwidth}
\centering
\includegraphics[width=.95\textwidth]{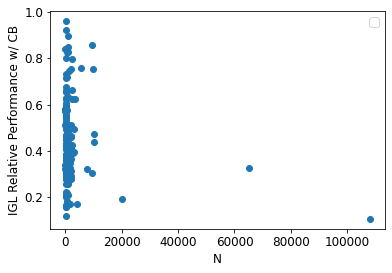}
\caption{}
\end{subfigure}
~
\begin{subfigure}{0.49\textwidth}
\centering
\includegraphics[width=.95\textwidth]{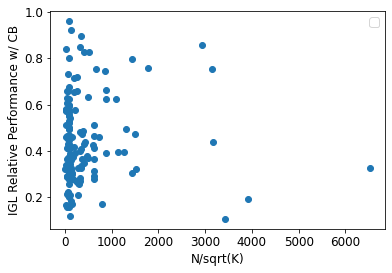}
\caption{}
\end{subfigure}
\caption{AI-IGL's relative performance w.r.t. CB versus different feature values for 155 datasets with $K\geq3$ and $N/K\leq1000$. $N/K$ itself is not the most important feature (Fig.~\ref{fig:importance_b}) for this subset of datasets, but the maximum Fisher discriminant ratio is which shows an approximately linear trend (Figure~\ref{fig:features_1000_d}).}
\label{fig:features_1000}
\end{figure*}

\begin{figure*}[t!]
\begin{subfigure}{0.49\textwidth}
\centering
\includegraphics[width=0.95\textwidth]{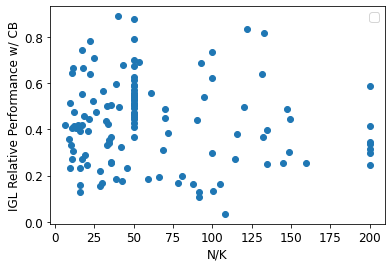}
\caption{}
\end{subfigure}
~
\begin{subfigure}{0.49\textwidth}
\centering
\includegraphics[width=.95\textwidth]{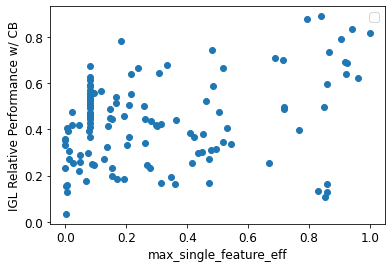}
\caption{}
\end{subfigure}
\\
\begin{subfigure}{0.49\textwidth}
\centering
\includegraphics[width=.95\textwidth]{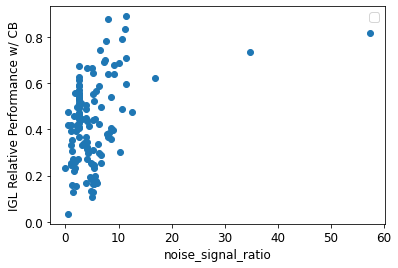}
\caption{}
\end{subfigure}
~
\begin{subfigure}{0.49\textwidth}
\centering
\includegraphics[width=.95\textwidth]{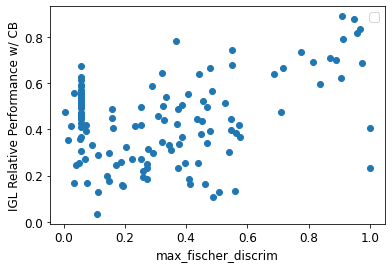}
\caption{}
\label{fig:features_200_d}
\end{subfigure}
\\
\begin{subfigure}{0.49\textwidth}
\centering
\includegraphics[width=.95\textwidth]{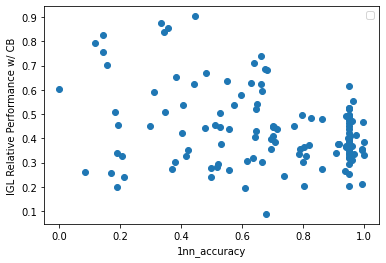}
\caption{}
\end{subfigure}
~
\begin{subfigure}{0.49\textwidth}
\centering
\includegraphics[width=.95\textwidth]{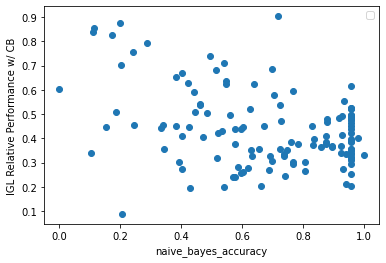}
\caption{}
\end{subfigure}
\\
\begin{subfigure}{0.49\textwidth}
\centering
\includegraphics[width=.95\textwidth]{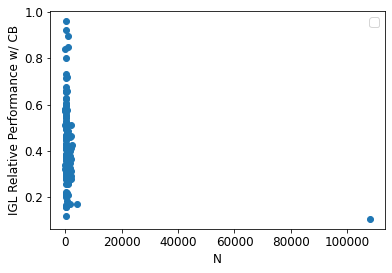}
\caption{}
\end{subfigure}
~
\begin{subfigure}{0.49\textwidth}
\centering
\includegraphics[width=.95\textwidth]{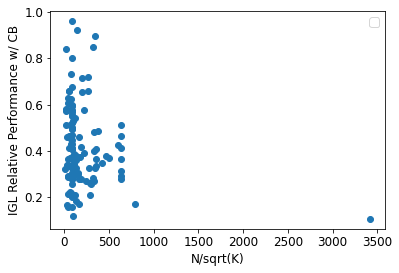}
\caption{}
\end{subfigure}
\caption{AI-IGL's relative performance w.r.t.~CB versus different feature values for 128 datasets with $K\geq3$ and $N/K\leq200$. $N/K$ itself has a small importance value (Figure~\ref{fig:importance_c}) for this subset of datasets. The maximum Fisher discriminant ratio is again the most important feature for this subset of datasets (showing an approximately linear trend Figure~\ref{fig:features_200_d}) to predict the relative performance of AI-IGL.}
\label{fig:features_200}
\end{figure*}

\end{document}